\documentclass[conference,compsoc]{IEEEtran}
\IEEEoverridecommandlockouts
\usepackage{tikz}
\usepackage{amsmath}
\usepackage{booktabs}
\usepackage{multirow}
\usepackage{algorithm}
\usepackage{algpseudocode}
\usepackage{xcolor}
\usepackage[colorlinks=true,linkcolor=red,citecolor=green]{hyperref}
\usepackage{amsmath}
\usepackage{amssymb}
\usepackage{amsthm}
\usepackage{caption}
\usepackage{tabularx} 
\usepackage{graphicx}

\usepackage{tcolorbox}

\tcbuselibrary{theorems}
\newtcolorbox{simplegreybox}{
  colback=gray!10,
  colframe=black,
  boxrule=0.5pt,
  arc=4pt,
  left=5pt,
  right=5pt,
  top=5pt,
  bottom=5pt,
  before skip=10pt,
  after skip=10pt,
}
\newtheorem{definition}{Definition}
\newtheorem{theorem}{Theorem}

\usepackage{fontawesome5}  
\usepackage{url} 

\usepackage{filecontents}
\usepackage{threeparttable}
\usepackage{cite}
\usepackage{pifont} 

\newcommand{\xmark}{\ding{55}}   
\newcommand{\bcircle}{\ding{108}}  
\newcommand{\ocircle}{\ding{109}} 

\usepackage{tikz}

\newcommand{\half}{
    \begin{tikzpicture}[baseline=-0.5ex]
        \draw[fill=black] (0,0) circle (1mm);
        \fill[white] (0,0) -- (270:1mm) arc[start angle=270, end angle=90, radius=1mm] -- cycle;
    \end{tikzpicture}
}
\ifCLASSINFOpdf
\else
\fi
\begin{document}
%
\title{ARES: Scalable and Practical Gradient Inversion Attack in Federated Learning through Activation Recovery}

\author{
    \IEEEauthorblockN{
        Zirui Gong\textsuperscript{1}, 
        Leo Yu Zhang\textsuperscript{\faEnvelope 1}, 
        Yanjun Zhang\textsuperscript{1}, 
        Viet Vo\textsuperscript{2}, 
        Tianqing Zhu\textsuperscript{3}, 
        Shirui Pan\textsuperscript{1},
        Cong Wang\textsuperscript{4}
    }
    \IEEEauthorblockA{\textsuperscript{1} Griffith University} 
    \IEEEauthorblockA{\textsuperscript{2} Swinburne University of Technology}
    \IEEEauthorblockA{\textsuperscript{3} City University of Macau}
    \IEEEauthorblockA{\textsuperscript{4} City University of Hong Kong}
    \thanks{\textsuperscript{\faEnvelope} Correspondence to Leo Yu Zhang \href{mailto:leo.zhang@griffith.edu.au}{(leo.zhang@griffith.edu.au).}}

}


%


\maketitle



\begin{abstract}
Federated Learning (FL) enables collaborative model training by sharing model updates instead of raw data, aiming to protect user privacy. However, recent studies reveal that these shared updates can inadvertently leak sensitive training data through gradient inversion attacks (GIAs). Among them, active GIAs are particularly powerful, enabling high-fidelity reconstruction of individual samples even under large batch sizes. Nevertheless, existing approaches often require architectural modifications, which limit their practical applicability.
In this work, we bridge this gap by introducing the \underline{A}ctivation \underline{RE}covery via \underline{S}parse inversion (ARES) attack, an active GIA designed to reconstruct training samples from large training batches without requiring architectural modifications. Specifically, we formulate the recovery problem as a noisy sparse recovery task and solve it using the generalized Least Absolute Shrinkage and Selection Operator (Lasso). To extend the attack to multi-sample recovery, ARES incorporates the imprint method to disentangle activations, enabling scalable per-sample reconstruction. We further establish the expected recovery rate and derive an upper bound on the reconstruction error, providing theoretical guarantees for the ARES attack.
Extensive experiments on CNNs and MLPs demonstrate that ARES achieves high-fidelity reconstruction across diverse datasets, significantly outperforming prior GIAs under large batch sizes and realistic FL settings. Our results highlight that intermediate activations pose a serious and underestimated privacy risk in FL, underscoring the urgent need for stronger defenses.
\end{abstract}

%
\IEEEpeerreviewmaketitle




%
\section{Introduction}
\begin{table*}[ht]
\centering
\caption{Comparison of different GIAs. \bcircle~ supported; \ocircle~ not supported; \protect \half partially supported.}
\label{tab:attack_comparison}
\resizebox{\textwidth}{!}{%
\begin{tabular}{ccccccccc}
\toprule
\textbf{Methods} & \textbf{Attack Type}  & \textbf{Large Batch} & \textbf{Architecture Integrity} & \textbf{One-shot} & \textbf{Complex Data}  & \textbf{Theoretical Guarantee}  \\
\midrule
iDLG \cite{zhao2020idlg}& Passive & \ocircle & \bcircle & \bcircle & \ocircle  & \ocircle \\
InvertingGrad (NeurIPS 2020) \cite{geiping2020inverting}&Passive  & \ocircle & \bcircle & \bcircle & \ocircle  & \ocircle \\
GradInversion (CVPR 2021) \cite{yin2021see}&Passive  & \ocircle & \bcircle & \bcircle & \ocircle  & \ocircle \\
FedLeak (USENIX 2025) \cite{fan2025boosting}&Passive & \ocircle & \bcircle & \bcircle & \bcircle  & \ocircle \\
Fishing (ICML 2023) \cite{wen2022fishing}&Active & \makebox[1.5em][c]{\half}\hspace{-0.4em}{$^*$} & \bcircle & \ocircle & \bcircle  & \ocircle \\
RtF (ICLR 2022) \cite{fowlrobbing}&Active & \bcircle & \ocircle & \bcircle & \bcircle  & \bcircle \\
Trap Weight (EuroS\&P 2023) \cite{boenisch2023curious}&Active & \bcircle & \makebox[1.6em][c]{\half}\hspace{-0.4em}$^\dagger$ & \bcircle & \bcircle &   \ocircle\\
LOKI (S\&P 2024) \cite{zhao2024loki}& Active & \bcircle & \ocircle & \bcircle & \bcircle  & \bcircle \\
\textbf{ARES (Ours)}& Active  & \bcircle & \bcircle & \bcircle & \bcircle &  \bcircle \\
\bottomrule
\end{tabular}}
\begin{tablenotes}
\item * Fishing \cite{wen2022fishing} can only recover a single image from a batch.  $^\dagger$ Trap Weight \cite{boenisch2023curious} requires the input to be positive to preserve architecture integrity.
\end{tablenotes}
\end{table*}

Federated Learning (FL) \cite{mcmahan2017communication} is a decentralized paradigm for training machine learning models, wherein multiple clients collaboratively optimize a shared global model without disclosing their private data. In this framework, each client preserves its local dataset and performs model training independently. The locally computed updates or gradients are then transmitted to a central server, which aggregates them to refine the global model and subsequently redistributes it to the clients for the next training round. FL has gained great popularity and has been widely adopted in healthcare and finance, as it enables the training of machine learning models on large-scale datasets without exposing clients’ raw data to the central server \cite{nguyen2022federated,mammen2021federated,gong2025not,gong2023agramplifier,chen2026beyond}.

However, recent studies demonstrate that FL can provide a false sense of privacy, as the server can extract sensitive information, including the clients’ private training data, from the shared updates.
Such threats, referred to as \textit{gradient inversion attacks} (GIA) \cite{zhu2019deep, zhao2020idlg, geiping2020inverting, yin2021see, fan2025boosting, wen2022fishing, fowlrobbing, boenisch2023curious, zhao2024loki,feng2024uncovering}, can be broadly classified into two categories, as summarized in Table \ref{tab:attack_comparison}. In the passive setting, an honest-but-curious server attempts to reconstruct training samples by iteratively minimizing the discrepancy between the gradients of dummy samples and the observed gradients of the true data \cite{zhu2019deep,zhao2020idlg, geiping2020inverting, yin2021see, fan2025boosting}. While conceptually straightforward, these approaches often fail to generalize to complex datasets and suffer substantial performance degradation when the training batch size increases (usually limited to batch sizes smaller than 64).

In contrast, active attacks assume a malicious server that manipulates the model parameters or architecture to amplify privacy leakage \cite{wen2022fishing, fowlrobbing, boenisch2023curious, zhao2024loki}. A commonly used strategy is the \textit{linear layer leakage}, where if only a single sample $x$ activates neuron $i$ in a fully connected (FC) layer, the sample $x$ can be directly revealed by dividing the weight gradient by the bias gradient of neuron $i$.
Building on this principle, the adversary designs malicious parameters to disentangle individual sample contributions across neurons, ensuring that each neuron’s recovery corresponds to a distinct training sample.
Despite their effectiveness, these attacks suffer from notable practical limitations, as they typically require modifications to the model architectures.
First, linear layer leakage can only reconstruct the direct inputs to an FC layer. Consequently, prior attacks typically insert a specially designed FC layer before the target model to enable reconstruction \cite{fowlrobbing,zhao2024loki}, resulting in a nonstandard architecture that may raise suspicion. Second, even for those networks that have an FC layer in the first layer (e.g., MLP), achieving high recovery rates demands the number of output neurons in the FL layer exceed the batch size, with about four times larger being optimal \cite{zhao2024loki}. 
This requirement forces adjustments to the network dimensions relative to the batch size, increasing the extent of structural modification and reducing the attacks' practicality.

\textbf{Our work:}
We bridge this gap and propose the \underline{A}ctivation \underline{RE}covery via \underline{S}parse inversion (ARES) attack, a practical and effective active GIA that scales to realistic large batch sizes without necessitating architectural modifications, thereby enhancing its applicability to real-world scenarios.
To achieve this, we first reveal that the practical limitations of prior active GIAs can be overcome by tackling the fundamental challenge of inverting hidden activations into training samples. Even without architectural modifications, the attacker can still leverage existing FC layers—typically located in deeper parts of the network—to obtain activations that are fed into them. These activations can then be further inverted to recover the corresponding training samples.

Motivated by this observation, this work focuses on the inversion problem from activations to training samples. To tackle this, we first formulate the operations before the FC layer as a linear transformation followed by a nonlinear transformation (caused by activation functions). Under this formulation, we identify two key challenges that hinder recovery.
First, the nonlinear activations in earlier layers render the overall mapping non-invertible, thereby precluding exact recovery of the training sample through a direct inverse transformation. Second, the linear transformation is often underdetermined (i.e., the number of unknowns exceeds the number of known measurements), particularly in MLP-based architectures. For example, sample features typically lie in a high-dimensional space (e.g., 14,784 for ImageNet), while the activation in the FC layer resides in a much lower-dimensional space, equal to the number of output neurons. Even in CNN-based networks, ReLU activations before the FC layer may zero out many measurements, discarding information that could aid recovery. Consequently, the available measurements are insufficient compared to the dimensionality of the target sample, making the reconstruction of training data from activations an  \textit{underdetermined nonlinear inversion problem}.

To address this problem, we relax the inversion task and reformulate it as a \textit{noisy sparse recovery problem}, and use the knowledge from compressed sensing theory to solve it \cite{1580791,candes2006stable,candes2008restricted,plan2016generalized,zhang2016bi}. Specifically, to mitigate the challenges posed by nonlinearity, we approximate the nonlinear transformation as a noisy, scaled linear mapping. To further address the underdetermined nature of the problem, we exploit the fact that many types of data (e.g., natural images, text embeddings, and audio signals) admit sparse representations in suitable domains. In other words, such data can be effectively compressed into a vector with most entries being zero, reducing the number of unknowns to recover. 
Based on these observations, we reformulate the problem as a noisy sparse recovery task and employ the generalized Least Absolute Shrinkage and Selection Operator (Lasso) method \cite{plan2016generalized} to identify the sparsest solution consistent with the observed measurements.
To extend recovery from single-sample to a batch of samples, we integrate the \textit{imprint method} proposed by RtF \cite{fowlrobbing}. In particular, we leverage the bias of the FC layer as cut-offs so that different samples activate different neurons. As a result, each neuron primarily captures the contribution of a single sample, enabling linear layer leakage to recover the corresponding activation.
Here, we do not require the number of neurons in the FC layer to exceed the batch size, as we allow the second layer separation. 
Finally, by invoking the recovery guarantees provided by the Restricted Isometry Property (RIP) \cite{candes2008restricted}, we establish a theoretical upper bound on the reconstruction error.

We evaluate our attack across five image datasets, including MNIST \cite{lecun1998mnist}, CIFAR-10 \cite{krizhevsky2009learning}, ImageNet \cite{deng2009imagenet}, HAM10000 \cite{tschandl2018ham10000}, Lung-Colon Cancer \cite{borkowski2019lung}, one text dataset (Wikitext \cite{merity2016pointer}), and one audio dataset (AudioMNIST \cite{becker2024audiomnist}) using representative CNN and MLP architectures. 
Our results demonstrate that our ARES  consistently outperforms all state-of-the-art attacks, achieving up to 7~$\times$ improvement in PSNR across various datasets and batch sizes.
Furthermore, we assess the robustness of our attack under five defense strategies, including differential privacy (DP) \cite{geyer2017differentially}, gradient quantization \cite{yue2023gradient}, gradient sparsification \cite{yue2023gradient}, data augmentation \cite{gao2023automatic} and secure aggregation \cite{bonawitz2017practical,fereidooni2021safelearn}, showing that ARES remains effective in these protected settings. Our key contributions can be summarized as follows:

\begin{itemize}
\item We reveal the practical limitation of the existing active GIA lies in the unsolved challenge of inverting hidden activations into training samples.
Based on the observation, we formulate the inversion task as a noisy sparse recovery problem and leverage principles from compressed sensing to solve it.

\item We propose ARES, a practical and effective active GIA that scales to realistic large batch sizes without requiring architectural modifications. ARES achieves this by exploiting linear layer leakage to extract intermediate activations and leveraging sparse recovery techniques to reconstruct the training samples from the extracted activations.

\item We provide a theoretical upper bound on recovery error and conduct extensive experiments\footnote{Our code is available at 
\href{https://github.com/gongzir1/ARES}{\textbf{https://github.com/gongzir1/ARES}}.
} on image, text, and audio datasets, demonstrating that ARES consistently outperforms state-of-the-art attacks by up to 7~$\times$ in PSNR across different settings.

\end{itemize}

\section{Preliminary}
\subsection{Gradient Inversion Attacks}
Passive GIAs \cite{zhu2019deep,zhao2020idlg,geiping2020inverting,yin2021see,fan2025boosting,feng2026mitigating} assume an honest-but-curious server or an external adversary with access to the model and individual gradients from each client. The attacker tries to minimize the difference between the observed ground-truth gradient and the gradient generated by the dummy sample, thereby optimizing the dummy sample to approximate the original input. Formally, the reconstruction can be formulated as
\begin{equation}
    \tilde{x} = \arg\min_x \left\| \nabla \mathcal{L}(x) - g \right\|^2,
\end{equation}
where $x$ is the dummy sample, $\tilde{x}$ is the reconstructed sample, $\mathcal{L}$ is the loss function, and $g$ is the observed gradient.
Recent works enhance this optimization by incorporating various regularizers, such as total variation (TV)~\cite{geiping2020inverting}, or some image priors tailored to natural image distributions~\cite{yin2021see}, to improve visual fidelity. These methods often yield good reconstruction results only when the batch size is small and the dataset is relatively simple. However, as the batch size increases, gradient contributions from different samples become entangled, making the optimization landscape more complex and the reconstruction less accurate.

By contrast, active GIAs assume a malicious server that can modify the model parameters or architecture to launch a stronger attack \cite{fowlrobbing,boenisch2023curious,wen2022fishing,zhao2024loki}. A commonly used strategy is the \textit{linear layer leakage}, where if only a \textit{single} sample $x$ activates neuron $i$ in an FC layer, the sample $x$ can be directly revealed by solving
\begin{equation}
    x = \frac{\partial \mathcal{L}}{\partial W_i} / \frac{\partial \mathcal{L}}{\partial b_i},
\end{equation}
where $\frac{\partial \mathcal{L}}{\partial W_i}$ and $\frac{\partial \mathcal{L}}{\partial b_i}$ are the gradients of the loss $\mathcal{L}$ with respect to the weight and bias for neuron $i$, respectively. 
To recover a batch of samples, RtF \cite{fowlrobbing} introduces the imprint method, which encourages each sample in the batch to leave a distinct imprint on a specific neuron, thereby enabling each neuron to be reverted to reveal individual samples.
However, this approach only enables the recovery of inputs fed directly into the FC layer. Consequently, reconstructing the original training samples requires placing the imprint module (the specially designed FC layer) at the beginning of the target model, leading to a non-standard architecture that could raise suspicion on the client side.

Trap Weight \cite{boenisch2023curious} attempts to overcome this limitation by leveraging the existing FC layer within the network for training sample reconstruction, thereby eliminating the need for architectural modifications. It initializes the weight matrices of all layers preceding the FC layer to act as direct-pass (identity) mappings, allowing inputs to propagate through unchanged to the FC layer. 
Then, the linear layer leakage can directly reveal the training samples.
LOKI \cite{zhao2024loki} extends this idea and proposes an attack targeting secure aggregation–based FL. In this setup, each client receives the model with distinct parameter configurations, where a subset of kernels (e.g., three) is set as direct-pass mappings and the remaining kernels are set to zero. This client-specific configuration prevents weight gradient from mixing across clients, thereby enabling large-scale recovery.
However, both methods are effective only when the network processes nonnegative inputs. Under standard settings, where inputs are normalized to follow $\mathcal{N}(0,1)$, ReLU activations in the preceding layers suppress negative values, thereby breaking the intended identity mapping and leading to information loss.
Scale-MIA \cite{shi2023scale} also leverages the model’s built-in FC layer to conduct an attack without modifying the architecture. However, it requires an auxiliary dataset (i.e., a subset of the training dataset) to train a decoder that maps latent representations back to the original samples, which limits its ability to generalize to unseen domains.
Detailed descriptions of attacks mentioned in Table \ref{tab:attack_comparison} are provided in Appendix \ref{app: attacks}.

\subsection{Defenses Against Gradient Inversion Attacks}
Defenses against GIA can be classified into three main categories: gradient perturbation-based methods \cite{geyer2017differentially,yue2023gradient}, data augmentation-based defense \cite{gao2021privacy, gao2023automatic}, and secure aggregation-based methods \cite{aono2017privacy,bonawitz2017practical}.
Gradient perturbation-based methods modify gradients sent to the server to avoid directly leaking training sample-related information to the server. 
Differential privacy (DP) \cite{geyer2017differentially} perturbs ground-truth gradients by adding random noise. Gradient sparsification \cite{yue2023gradient,xue2024revisiting} transmits only the most significant gradient elements, while gradient quantization \cite{yue2023gradient} reduces precision by representing gradient values with fewer bits. Although effective, these strategies incur a trade-off between model utility and privacy, as more substantial modifications yield better protection but degrade the gradient utility. 

Data augmentation–based defenses \cite{gao2021privacy, gao2023automatic} apply carefully chosen transformations to the training data to prevent adversaries from reconstructing both the augmented and original samples from shared gradients, while preserving model utility. The key idea is to disrupt the prior knowledge exploited by attackers, i.e., total variation or batch normalization statistics, that guides the reconstruction process.

Secure aggregation-based methods protect user training data by ensuring that the server can only access the plaintext of the aggregated gradients \cite{bonawitz2017practical,fereidooni2021safelearn}. This makes data reconstruction significantly more challenging, as the aggregated gradients correspond to a larger global batch size that must be recovered.
Detailed descriptions of defenses are provided in Appendix \ref{app: defenses}.

\subsection{Threat Model and Attack Scope}

\noindent\textbf{Threat Model.} 
We consider a \emph{malicious server} that controls the FL training process and can modify the weights and biases of the global model before distributing them to clients. Unlike prior works \cite{fowlrobbing, zhao2024loki}, the server cannot alter the network architecture or design a non-standard model to facilitate an attack. This assumption is realistic because, in FL, the model architecture and training protocol are typically agreed upon in advance, and unsolicited architectural changes are generally rejected or detected. Given that clients depend on the server for model distribution and without insight into its internal operations, it is reasonable to treat the server as potentially malicious \cite{wen2022fishing,wang2021field,fowl2022decepticons,shi2023scale}. 
Such an attack can also be executed by any party that obtains the server’s state, e.g., through a temporary breach \cite{fowl2022decepticons}.
The adversary's objective is to reconstruct as many distinct training examples as possible.


\noindent\textbf{Attack Scope.} 
We evaluate attacks on two widely used families of models: (i) CNN-based networks, consisting of convolutional layers followed by fully connected layers, and (ii) MLP-based networks, including fully connected layers.
We also assume that the clients' private training data admits a sparse representation in a suitable domain. 
This is reasonable, as most real-world data (e.g., natural images, text embeddings, and audio signals) can be sparsely represented in suitable domains \cite{6472238,5456194, 5452966}. 

\section{Method}

\subsection{Motivation}
Existing active gradient inversion attacks (aGIAs) exploit linear layer leakage to analytically reconstruct training samples from the gradients of an FC layer. By configuring malicious parameters, the attacker aims to have each neuron activated by a single sample (or to imprint each neuron with a single sample), so that the recovery from each neuron directly reveals individual samples
\cite{wen2022fishing, fowlrobbing, boenisch2023curious, zhao2024loki}. Compared to passive GIAs, aGIAs induce stronger privacy breaches and are more effective at recovering samples from large training batches.
Despite their effectiveness, a major criticism is their reliance on modifying the network architecture, which mainly stems from two factors.
First, linear layer leakage can only reveal the direct inputs to the FC layer. Thus, recovering the original training samples requires the FC layer to be the first layer of the model. This condition is not satisfied in most modern architectures, such as CNNs, which contain multiple convolutional layers before the FC layer.
To address this, Trap Weight \cite{boenisch2023curious} proposes initializing the layer preceding the FC layer with a direct-pass (identity-like) weight matrix to avoid value distortion, thereby aligning the FC-layer inputs with the model’s original inputs.
However, this approach works only for nonnegative inputs, as the ReLU activation zeros out negative values and results in information loss.
Consequently, existing works typically insert a specially designed FC layer before the target network to enable the training sample recovery \cite{fowlrobbing,zhao2024loki}.
Second, even for models where an FC layer is already the first layer (e.g., multi-layer perceptron–based networks), achieving a high recovery rate requires the number of output neurons to exceed the batch size \cite{fowlrobbing, zhao2024loki}, and in practice, four times larger than the batch size is optimal \cite{zhao2024loki}. Therefore, the attacker must modify the network’s dimensionality to satisfy this requirement.

However, we observe that both constraints can be overcome by addressing the fundamental challenge of inverting activations into training samples.
For the first constraint, even without model modification, we can leverage the built-in FC layer in standard networks to extract the individual activations that are fed into it.
For instance, consider a layer \(l > 1\) that is an FC layer in the network. If only a single sample activates neuron $i$, the corresponding activation \(h^{(l-1)}\) associated with neuron $i$ can be reconstructed as  
\begin{equation}
\label{eq: single_fc_recovery}
    h^{(l-1)} = \frac{\partial \mathcal{L}}{\partial W_i^{(l)}} \Big/ \frac{\partial \mathcal{L}}{\partial b_i^{(l)}},
\end{equation}
where \(\frac{\partial \mathcal{L}}{\partial W_i^{(l)}}\) and \(\frac{\partial \mathcal{L}}{\partial b_i^{(l)}}\) are the weight gradient and bias gradient of neuron $i$ in layer $l$ (see Appendix \ref{app: batch_activation_recovery} for a detailed derivation). Although feasible, the unresolved challenge lies in reconstructing the training samples from recovered activations.
For the second constraint, which mandates a large number of output neurons in the FC layer, this limitation can be mitigated by employing a subsequent FC layer to further disentangle samples that remain mixed in the first layer. In this case, it still need to address the challenge of reconstructing the training samples from activations (separated in the second layer).

Motivated by the above observations, in this work, we focus on the problem of reconstructing training samples from hidden activations.
To start with, we formalize the computation prior to layer \(l\) as
\begin{equation}
\label{eq: og_forward}
h^{(l-1)} = f(Wx + b),
\end{equation}
where \(f\) denotes a nonlinear transformation (caused by activation functions), and \(W\) and \(b\) represent the effective weight and bias of the linear transformation (i.e., weight and bias of the first layer), respectively.
While this formulation is straightforward, inverting it to recover $x$ is far from trivial for two main reasons.

\noindent\textbf{Challenges:}  
\begin{figure}
    \centering   \includegraphics[width=0.8\linewidth]{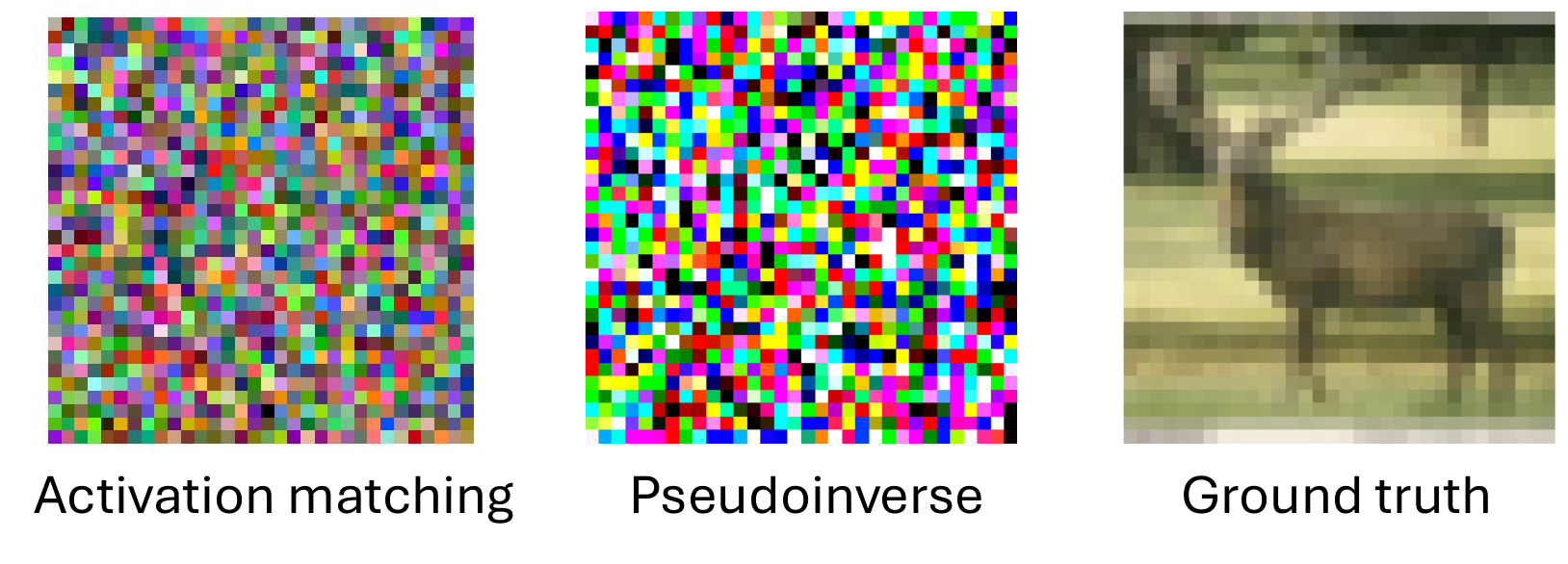}
    \caption{Recovery results obtained using activation matching (optimizing activation discrepancy) and pseudoinverse (approximating inverse).} 
    \label{fig:challenge}
\end{figure}
First, the nonlinear activations in preceding layers render the overall transformation non-invertible, preventing exact recovery of \(x\) via a direct inverse mapping. Second, the linear transformation in Eq.~\eqref{eq: og_forward} is often underdetermined, particularly in MLP-based networks. Specifically, sample features typically lie in a high-dimensional space (e.g., 14,784 for ImageNet), whereas the activation of the FC layer resides in a much lower-dimensional space. Even in CNN-based networks, ReLU activations preceding the FC layer may zero out many measurements, which discards the information that could aid recovery. As a result, the number of available measurements is insufficient relative to the dimensionality of the target sample. 
Consequently, recovering \(x\) from \(h^{(l-1)}\) constitutes an \textit{underdetermined nonlinear inversion problem}.

Solving this problem is inherently challenging. One possible approach is to optimize the discrepancy between ground-truth activations and the dummy activations to optimize the dummy samples that best match the original; we refer to this approach as \textit{activation matching}. However, unlike traditional gradient matching, which leverages gradients from all layers to guide the optimization, this method relies solely on activations from a single layer. As a result, the recovery process is highly ill-posed and often yields a suboptimal reconstruction effect (as shown in the left panel of Fig.~\ref{fig:challenge}). 
Another approach involves using the Moore–Penrose pseudoinverse of the weight matrix to approximate the inverse of $W$, and only using the linear part of the activation function (i.e., the region where ReLU is active and the output equals the input) to calculate the value of $x$. 
Nevertheless, because the system remains underdetermined, the solution is not unique; therefore, this method also fails to reconstruct the original samples faithfully (as shown in the middle panel of Fig.~\ref{fig:challenge}).



\subsection{Overview}
\begin{figure*}
    \centering   \includegraphics[width=0.8\textwidth]{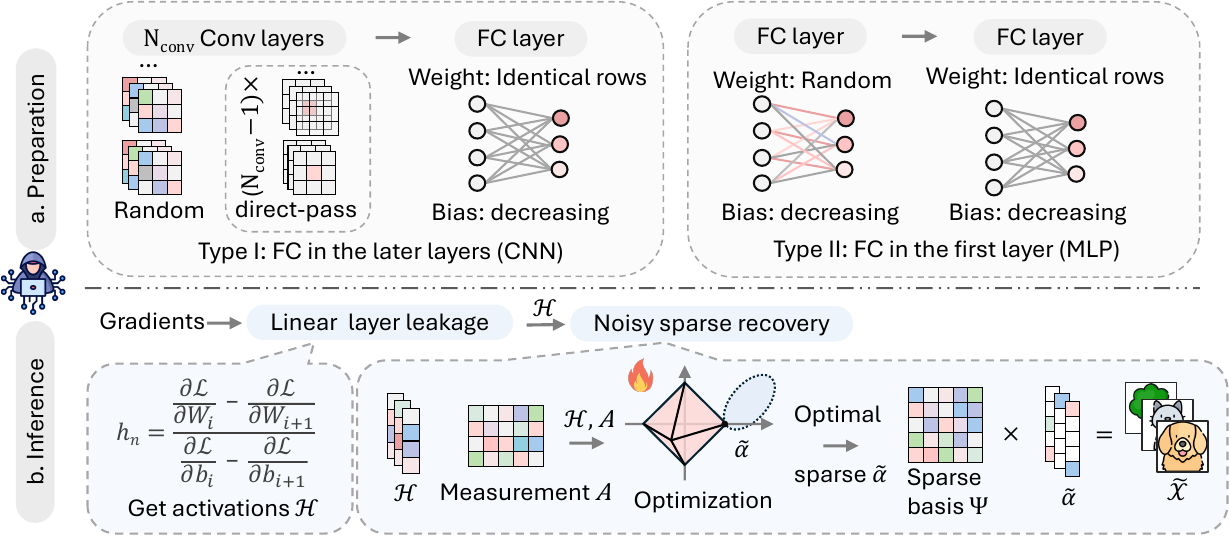}
    \caption{Overview of ARES attack. The method consists of two main stages: (a) the attacker initializes network with malicious parameters to facilitate information leakage; (b) using gradients returned by the client, the attacker first recovers activations through linear layer leakage and then reconstructs input samples via noisy sparse recovery. 
    }
    \label{fig: overall}
\end{figure*}
Based on the above observations, this work focuses on addressing the underdetermined nonlinear inversion problem to recover training samples from activations, thereby enabling an effective and practical GIA. To achieve this, the malicious server (hereafter referred to as the attacker) operates in two main stages: preparation and inference, as illustrated in Fig.~\ref{fig: overall}. During the preparation stage, the attacker configures the network parameters to maximize information leakage. Specifically, for the FC layer, it designs malicious weights and biases such that each neuron leaves a distinct imprint corresponding to a single sample, effectively isolating one sample per neuron. For layers preceding the FC layer (if any), the attacker sets the weight parameters to provide sufficient and non-redundant measurements, facilitating accurate reconstruction of training samples from recovered activations. Once configured, the attacker sends malicious parameters to clients for local training.

After clients complete local training and return their updates to the server, the attacker proceeds to the inference stage. The attacker first leverages linear layer leakage by using the gradients of the FC layer to recover either individual training samples or the activations fed into the FC layer. If the recovery directly yields the training samples, they are retained; otherwise, the attacker reconstructs the samples from the recovered activations. To handle the underdetermined and nonlinear nature of this reconstruction task, the attacker reformulates it as a noisy sparse recovery problem and solves it using the generalized Lasso method \cite{plan2016generalized}. By leveraging the recovery guarantees provided by the Restricted Isometry Property (RIP) \cite{candes2008restricted}, we further derive a theoretical upper bound on the recovery error.

In the following sections, we first provide an explanation of how to recover a single training sample from its activation under underdetermined and nonlinear transformations (Section \ref{sec: noisy_sparse_recovery}). We then extend this approach to enable the recovery of batches of samples (Section \ref{sec: muti-sample}). Finally, we present attack implementations for CNN and MLP networks and analyze the expected recovery rate (Section \ref{sec: final_summary}).

\subsection{Noisy Sparse Recovery}
\label{sec: noisy_sparse_recovery}
In this section, we focus on recovering a single training sample from activation under underdetermined and nonlinear measurement.
To make the problem tractable, we first relax the nonlinear mapping as a noisy and scaled linear mapping, which is a commonly used technique in the literature \cite{plan2016generalized,genzel2019recovering,genzel2016high}.
Specifically, we assume $f_i(\xi) = \mu \xi + z_i$, where $\mu$ is a scaling factor that captures the linear component of the $f$ and $z_i \sim \mathcal{N}(0, \sigma^2)$ is the noise term. 
To quantify the linear component and nonlinearity of the function, we provide the following definition.
\begin{definition}[Linear Component and Nonlinearity of a Function \cite{plan2016generalized}]
\label{def: nonlinear}
Let $f:\mathbb{R}^d \to \mathbb{R}^k$ be a nonlinear function, and let $\xi \sim \mathcal{N}(0,1)$ be a standard Gaussian random variable.  
The effective linear component of $f$ is defined as
\begin{equation}
\mu := \frac{1}{k} \sum_{i=1}^{k} \mathbb{E}[f_i(\xi) \xi],
\end{equation}
which represents the average linear component of $f$ across $k$ dimensions.
The residual nonlinearity is quantified by
\begin{equation}
\sigma^2 := \frac{1}{k} \sum_{i=1}^{k} \mathbb{E}[(f_i(\xi) - \mu \xi)^2], \quad
\eta^2 := \frac{1}{k} \sum_{i=1}^{k} \mathbb{E}[(f_i(\xi) - \mu \xi)^2 \xi^2],
\end{equation}
which measures the variance of the nonlinear part and how it interacts with the input magnitude, respectively.
\end{definition}
Under this approximation, Eq.~\eqref{eq: og_forward} becomes\footnote{We omit the layer index for clarity; unless explicitly stated otherwise, we assume layer $l$ is the FC layer and $h$ is the input to layer $l$.} 
\begin{equation} 
\label{eq: appro_linear} 
h \approx \mu (W x + b),
\end{equation}
and the recovery problem reduces to a noisy, underdetermined linear inversion problem.
To further handle the underdetermined issue, we leverage the fact that many types of data (e.g., natural images, text embeddings, and audio signals) admit sparse representations in suitable domains \cite{5456194, 5452966}. In other words, these data can be effectively compressed into vectors with most entries being zero, thereby reducing the number of unknowns to be recovered. 
Under this condition,  we express 
$x = \Psi \alpha$, where \(\alpha\) is a sparse coefficient vector that encodes the essential information of \(x\), with most of its entries being zero and \(\Psi\) is a sparse basis matrix that maps the sparse representation \(\alpha\) to the original signal space.  To obtain the sparse basis $\Psi$ for the image and audio data, we apply Discrete Cosine Transform (DCT) compression, providing a data-independent basis that efficiently represents samples in the frequency domain.
For the text dataset, we learn a sparse basis from the public token embeddings matrix of the pretrained model, which more effectively captures the underlying sparse structure of the token representations.
Then, Eq. \eqref{eq: appro_linear} can be expressed as 
\begin{equation}
\label{eq: sparse_forward}
h \approx \mu A \alpha + \mu b,
\end{equation}
where \(A = W \Psi\) denotes the sensing matrix that acts directly on the sparse vector \(\alpha\) (see top of Fig. \ref{fig:sparse_recovery} for illustration).
This transformation enforces sparsity on the variable to be recovered, effectively reducing its degrees of freedom.

Nonetheless, achieving an exact and unique recovery of $\alpha$ from $h$ in Eq. \eqref{eq: sparse_forward} further requires the measurement $A$ to possess sufficient information. In particular, it should provide enough independent measurements and preserve the relative geometry among all sparse vectors, such that different sparse vectors yield distinct outcomes under the mapping $A$. Formally, this requirement is characterized by the Restricted Isometry Property (RIP) \cite{candes2008restricted}.

\begin{definition}[Restricted Isometry Property \cite{candes2008restricted}]
A matrix \(A\) is said to satisfy the \textit{Restricted Isometry Property (RIP)} of order \(s\) with constant \(\delta_s \in (0,1)\) if, for all \(s\)-sparse vectors \(\alpha\) (i.e., vectors with at most \(s\) non-zero entries),
\begin{equation}
\label{eq: rip}
(1 - \delta_s) \|\alpha\|_2^2 \leq \|A \alpha\|_2^2 \leq (1 + \delta_s) \|\alpha\|_2^2.
\end{equation}
Here, \(\delta_s\) is the \emph{restricted isometry constant} that quantifies how \(A\) preserves the Euclidean norm of all \(s\)-sparse vectors. 
\end{definition}
\noindent Eq. \eqref{eq: rip} means that multiplying a sparse vector \(\alpha\) by \(A\) changes its Euclidean norm by at most a factor of \((1 \pm \delta_s)\).
A smaller \(\delta_s\) indicates better preservation of the Euclidean norm of the sparse vector.
\begin{simplegreybox}
\noindent\textbf{Remark I.}
If a matrix \(A\) satisfies the \emph{Restricted Isometry Property} (RIP) of order \(s\), it approximately preserves the Euclidean norm of all \(s\)-sparse vectors.  
Moreover, if \(A\) satisfies the RIP of order \(2s\), it also approximately preserves the pairwise Euclidean distances between all $s$-sparse vectors, 
since the difference between any two $s$-sparse vectors is at most $2s$-sparse. 
In other words, the measurement $A\alpha$ preserves the geometric relationships among sparse signals, ensuring that distinct sparse inputs remain distinguishable after transformation.
Consequently, the information contained in $A\alpha$ is sufficient to enable exact recovery of the sparse vector $\alpha$ in the noiseless case, and stable recovery with small error when noise is present.
\end{simplegreybox}
\begin{figure}
    \centering   \includegraphics[width=\linewidth]{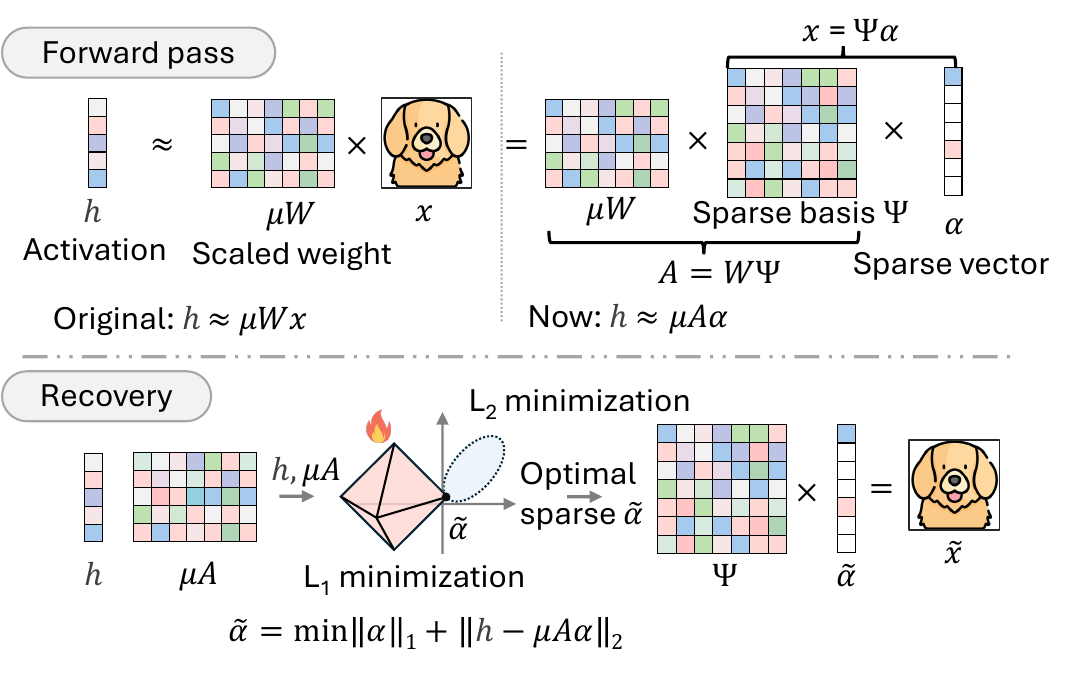}
    \caption{Top: forward pass through the network. Bottom: sparse vector recovery via $\ell_1$ optimization. The bias term is omitted for clarity.}   \label{fig:sparse_recovery}
\end{figure}

We now consider how to configure the malicious parameters such that the resulting measurement matrix $A$ satisfies the RIP, thereby enabling the exact recovery of $\alpha$ from $h$.
This can be achieved by initializing the weight $W$ as a Gaussian random matrix, since the product of a Gaussian random matrix and a sparse orthonormal basis $\Psi$ satisfies the RIP with high probability \cite{candes2005decoding}.
Building on the RIP condition, we focus on recovering the sparse vector~$\alpha$. Intuitively, our goal is to find the sparsest solution that satisfies Eq.~\eqref{eq: sparse_forward}. Formally, this can be expressed as solving

\begin{equation}
\label{eq:l0}
\tilde{\alpha}=\arg\min_{\alpha}  \| \alpha \|_0 + \| h-\mu A \alpha - \mu b \|_2,
\end{equation}
where the first term denotes the number of non-zero entries in \(\alpha\) and the second term quantifies the discrepancy between the observed activation $h$ and the activation reconstructed from the sparse vector $\alpha$.
Since solving Eq. \eqref{eq:l0} is NP-hard, a common relaxation is to replace the $\ell_0$ norm with the convex $\ell_1$ norm, 
\begin{equation}
\label{eq: l1}
\tilde{\alpha}=\arg\min_{\alpha}  \| \alpha \|_1 + \| h-\mu A \alpha - \mu b  \|_2,
\end{equation}
where $\|\alpha\|_1$ promotes sparsity while keeping the optimization tractable. 
We solve Eq.~\eqref{eq: l1} as a convex Lasso problem using the CVXPY framework \cite{diamond2016cvxpy}, which dispatches the underlying solver to compute the solution. This procedure recovers the sparsest vector $\alpha$ consistent with the observed activations.
Once $\tilde{\alpha}$ is obtained , it can be mapped back to the input space via  $\tilde{x}=\Psi \tilde{\alpha}$ (see bottom of Fig. \ref{fig:sparse_recovery} for illustration).
We now provide an upper bound on the recovery error of the solution to Eq.~\eqref{eq: l1}.

\begin{theorem}[Recovery Error \cite{plan2016generalized}]
\label{therm: error}
Let \( \alpha \) be an \(s\)-sparse vector, and let \(A\) be a measurement matrix satisfying the RIP of order \(2s\).  
Then, the solution $\tilde{\alpha}$ to Eq. \eqref{eq: l1} recovers $\alpha$ with the error bounded by
\begin{equation}
\label{eq: error_bound}
\varepsilon \sim \frac{\sqrt{s \log(d/s)} \, \sigma + \eta}{\sqrt{m}} 
+ |\mu - 1| \|\alpha \|_2,
\end{equation}
where $s$ is the sparsity of $\alpha$, $d$ is the ambient dimension of $\alpha$, $s \log(d/s)$ quantifies the effective dimension of the sparse signal,
and $m$ is the number of effective measurements (i.e., the number of non-redundant rows of $A$).
\end{theorem}
\begin{proof}
Theorem~\ref{therm: error} follows from Theorem 1.4 of \cite{plan2016generalized}, which provides an upper bound for $\|\tilde{\alpha} - \mu \alpha\|_2$. We then apply the triangle inequality to yield the upper bound for $\|\tilde{\alpha} - \alpha\|_2$ as stated in Eq. \eqref{eq: error_bound}.
\end{proof}



\begin{simplegreybox}
\noindent\textbf{Remark II.}
The first term in Eq. \eqref{eq: error_bound} captures the combined effect of the signal's sparsity and dimension, measurement noise, and residual nonlinearity on recovery error.
It indicates that the recovery error $\varepsilon$ decreases as the number of measurements $m$ increases and $\varepsilon$ increases with the effective dimension of the sparse signal increase.
Here, $s \log(d/s)$ reflects signal complexity, and sparser (smaller $s$) or lower-dimensional signals (smaller $d$) have smaller $s \log(d/s)$.
Furthermore, $\sigma$ and $\eta$ capture the effect of the nonlinear function $f$ on recovery. Larger values of $\sigma$ or $\eta$ correspond to higher variance or stronger nonlinearity of $f$, increasing recovery error.
The second term in Eq. \eqref{eq: error_bound} quantifies the error due to potential mismatch in scaling between $\mu \alpha$ and the true signal $\alpha$.
\end{simplegreybox}
\begin{figure}
    \centering
    \includegraphics[width=\linewidth]{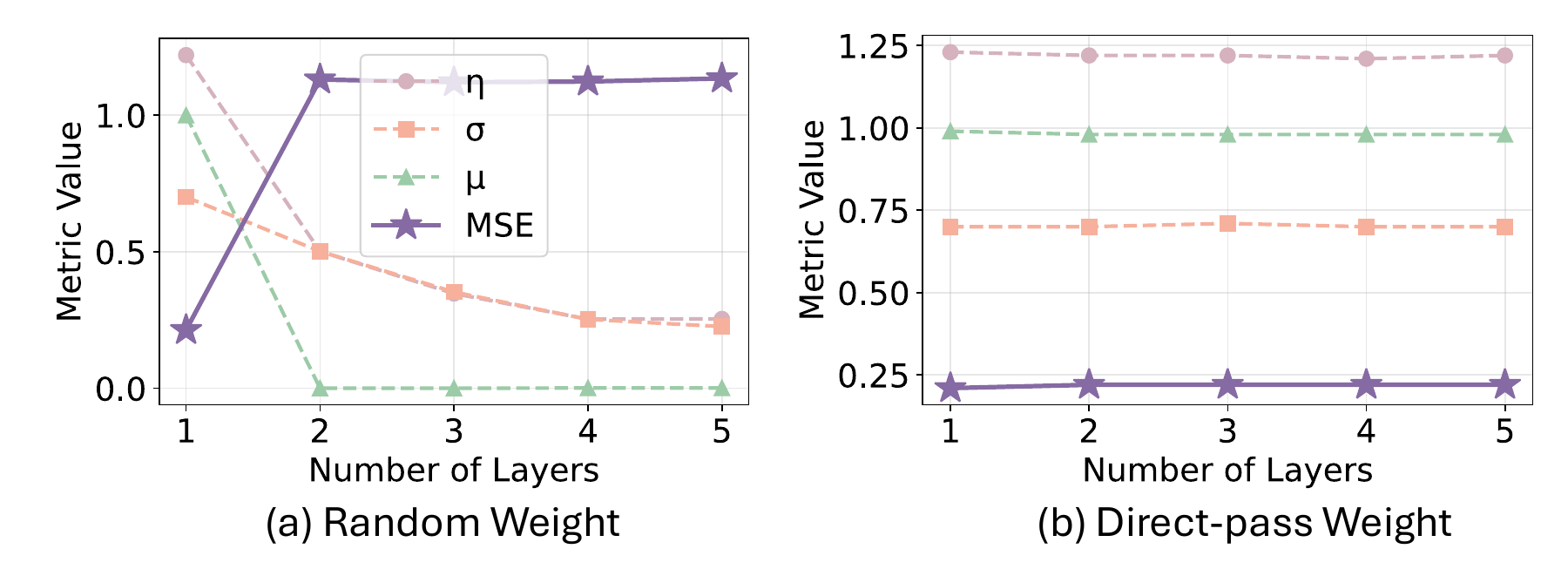}
    \caption{Upper bound of the squared recovery error with varying numbers of convolutional layers.} 
    \label{fig: recovery_error}
\end{figure}
Theorem~\ref{therm: error} indicates that, for a fixed number of measurements $m$, input dimension $d$, and sparsity $s$, the factors influencing recovery error are the nonlinear characteristics of the function $f$, i.e., $\mu$, $\sigma$, and $\eta$, as defined in Definition~\ref{def: nonlinear}.
Building on this observation, we further examine how the network architecture, particularly the number of layers preceding the FC layer, influences these parameters.
As illustrated in Fig.~\ref{fig: recovery_error}, which shows the empirical nonlinear values of $f$, alongside the squared recovery error (MSE) computed from Eq.~\eqref{eq: error_bound}, when the weights of each layer are randomly initialized (left in Fig. \ref{fig: recovery_error}), a network with only a single layer before the FC layer exhibits $\mu \approx 1$, indicating that the linear component of $f$ is reasonably captured and consequently the upper bound of the MSE remains acceptable (0.25).
However, from two layers onward, $\mu$ approaches zero, reflecting the increasingly nonlinear nature of $f$ (due to the activation functions). Meanwhile, the nonlinear residual decreases as the overall variance of $f$ diminishes with depth (due to ReLU masking).
As a result, linear approximations become less accurate, causing an increase in the upper bound of the recovery error. 

To mitigate this issue and enable effective recovery in deeper networks, we manipulate the weights of convolutional layers to control the nonlinearity of the function $f$. Specifically, we observe that a single activation function preserves acceptable linearity, but adding more layers with activation functions significantly increases nonlinearity, making the function harder to invert.
To prevent this accumulation, we adjust the convolutional weights \textit{from the second layer onward} using the direct-pass initialization method, so that these layers direct pass the inputs unchanged to the activation function. This ensures that all activation functions in the network operate on the same input, producing a consistent output rather than compounding nonlinear effects across layers. Practically, we implement direct-pass by setting the central element of each input channel to 1 and all other elements to 0, creating an identity-like mapping that preserves the input through the convolutional operation.
As shown in Fig.~\ref{fig: recovery_error} (right), the parameter $\mu$ remains stable across all layers under the direct-pass initialization, indicating that the network’s linear component is consistently preserved. Consequently, the upper bound of MSE remains stable across deeper layers, indicating the effectiveness of our method on deeper networks. 


\subsection{Multiple Samples Recovery}
\label{sec: muti-sample}
In this section, we extend our recovery method to a batch of samples.
The objective function in Eq.~\eqref{eq: l1} addresses the problem of recovering a single input sample \(x\) from the observed activation \(h\).  
However, when a batch of \(N\) samples is propagated through the network, the activation of neuron \(i\) in an FC layer reflects a weighted combination of contributions from all samples that activate neuron $i$.  
Consequently, Eq.~\eqref{eq: single_fc_recovery} is modified as
\begin{equation}
\label{eq: linear_recovery_batch}
\frac{g_i^{(W)}}{g_i^{(b)}} = \frac{\sum_{n=1}^{N_i} \gamma_i^n h_n}{\sum_{n=1}^{N_i} \gamma_i^n},
\end{equation}
where \(g_i^{(W)}\) and \(g_i^{(b)}\) denotes the weight and bias gradient of neuron \(i\), respectively; \(\gamma_i^n\) is the backpropagated error for neuron \(i\) in sample \(n\), \(h_n\) is the activation of sample \(n\) in layer \(l-1\) (i.e., the input to layer \(l\)), and \(N_i\) is the number of samples that activate neuron \(i\) (see Appendix \ref{app: batch_activation_recovery} for
detailed derivation).
Given \(k\) output neurons, we obtain \(k\) such equations (i.e., Eq.~\eqref{eq: linear_recovery_batch}).  
However, the total number of unknowns is \(N(k + 1)\), as each of $N$ samples contributes one activation variable $h_n$ and \(k\) backpropagation error \(\gamma_i^n\).  
Consequently, recovering all variables from these equations alone is impossible.
Fortunately, our objective is not to solve for all unknowns but only to recover the activations.
To achieve this, we adopt the \emph{imprint method}~\cite{fowlrobbing}, which enforces a structured activation pattern in the FC layer such that each neuron uniquely corresponds to the activation of a single sample.
Specifically, we configure the weight matrix in the FC layer to have identical rows (i.e., $W_i^{(l)}=w^{(l)}$ for all $i$, where $w^{(l)}$ is a constant vector), so each output neuron receives the same weighted combination of inputs.
Based on the distribution of $w^{(l)} h$, we adjust the biases to partition the input space into $k$ bins, maximizing the likelihood that each activation falls into a distinct bin.

\begin{figure}
    \centering
    \includegraphics[width=0.9\linewidth]{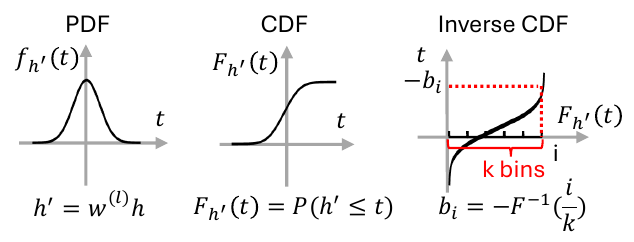}
    \caption{Bias values are set to divide the projected inputs into $k$ equal-probability bins.} 
    \label{fig: cdf}
\end{figure}

As shown in Fig. \ref{fig: cdf}, to divide the input space into \(k\) equal mass bins, we first derive the probability density function (PDF) of \(w^{(l)} h\).
Here, no prior knowledge of the dataset is required; we simply assume that the network input is normalized to follow a standard normal distribution, i.e., $x \sim \mathcal{N}(0,1)$, which is a reasonable assumption for most training setups. Given the value of malicious weight, we can then obtain the corresponding PDF of \(w^{(l)} h\).
Then we compute its cumulative distribution function (CDF) and partition it into \(k\) intervals of equal probability, so that each bin has approximately the same likelihood of containing a sample.
Specifically, let \(F(\cdot)\) denote the CDF of \(w^{(l)} h\). To partition the space into \(k\) equal-probability bins, we set the bias of each neuron as
\begin{equation}
\label{eq: set_bin}
b_i = -F^{-1}\Bigl( \frac{i}{k} \Bigr), \quad i = 1, \dots, k.
\end{equation}
where $F^{-1}$ is the inverse CDF, which maps uniform probability intervals back to the original value space.

Once clients train on the configured weights and biases, each activation $h_n$ is likely to fall into a distinct bin.
When an activation $h_n$ enters a bin, e.g., the interval $[-b_i,-b_{i+1})$, it activates neuron $i$ as well as all neurons associated with larger biases.
This yields a triangular (or progressive) activation pattern, where among the activated neurons, the one with the smallest bias is activated by a single sample, followed by the next neuron, which is activated by two samples, and so on.
This pattern naturally motivates a recursive elimination strategy in which we first invert the neuron activated by a single sample, subtract its contribution from the remaining neurons, and then iteratively isolate the activations of all other samples.
Formally, individual activations can be recovered by iteratively solving
\begin{equation}
\label{eq: objective} 
h_n=\frac{\gamma_i^n g_{i}^{(W)} - \gamma_{i+1}^n g_{i+1}^{(W)}}
{\gamma_i^n g_{i+1}^{(b)} - \gamma_{i+1}^n g_{i+1}^{(b)}}.
\end{equation}
However, Eq.~\eqref{eq: objective} cannot be solved directly because the values \( \gamma_i^n \) are unknown. 
A straightforward approach is to assume that all $\gamma_i^n$ are identical for each sample, i.e., $\gamma_i^n = \gamma^n$ for all $i$, which allows $\gamma^n$ to be eliminated from the equation.
This can be achieved if the next layer has identical columns in the weight matrix \cite{fowlrobbing}, i.e., \(W_{j}^{(l+1)}=w^{(l+1)} \) for all \( j \) columns (see Appendix~\ref{app: identical_columns} for a detailed derivation).
However, this constraint can be further relaxed by leveraging reasonable auxiliary knowledge, assuming that the attacker can infer the ground-truth label once the sample gets recovered.
In this case, the full gradient signal \( \{\gamma_i^n\}_{i=1}^k \) can be computed by feeding the recovered samples and their corresponding labels into the network. This approach is useful when only a single FC layer is present in the model.
Once the attacker obtains individual activations, it solves Eq. \eqref{eq: l1} to recover all $\tilde{\alpha}$ and reconstruct $\tilde{x}$ for the batch.

\subsection{Attack Implementation and Recovery Rate}
\label{sec: final_summary}
In this section, we illustrate the implementation of our attack under CNN and MLP-based networks.
Our attack for the CNN-based network is summarized in Algorithm~\ref{alg:attack_CNN}. Specifically, in the preparation stage, the attacker first initializes the convolutional kernels: the first kernel is initialized using a Gaussian random distribution (line~2), while the remaining kernels are initialized using the direct-pass method (line~3). For the FC layer, the weight matrix $W^{(l)}$ is initialized with identical rows (line~4). Based on $W^{(l)}$, the attacker derives the CDF of the projection values (line~5) and assigns the biases as the negative quantiles of this distribution, effectively disentangling the contribution of samples (lines~6--8).
In the inference stage, the attacker first reconstructs the set of activations $\mathcal{H}$ from the gradients in the FC layer using linear layer leakage (line~10). Then, for each activation $h_n \in \mathcal{H}$, the attacker computes a sparse coefficient vector $\tilde{\alpha}_n$ (line~12). Finally, the training sample $\tilde{x}_n$ is recovered by projecting $\tilde{\alpha}_n$ back into the input space using the sparse basis $\Psi$ (line~13). 
\begin{algorithm}[!h]
\caption{ARES for CNN Network.}
\label{alg:attack_CNN}
\begin{algorithmic}[1]
\renewcommand{\algorithmicrequire}{\textbf{Input:}}
\renewcommand{\algorithmicensure}{\textbf{Output:}}

\Require Weight gradient $g^{(W)}$ and bias gradient $ g^{(b)}$,  sparse basis $\Psi$, FC layer $l$ with $k$ output neurons
\Ensure A set of recovered training samples $\tilde{\mathcal{X}}$

\State \textbf{// Attacker Preparation}
\State $K^{(1)} \gets$ \text{Gaussian random weight}  \Comment{first conv layer}
\State $K^{(2)}, K^{(3)}, \dots, K^{(l-1)} \gets \text{direct-pass weight}$ \Comment{remaining conv layers}
\State $W^{(l)} \gets$ \text{identical rows}  \Comment{FC layer weights}
\State $F \gets$ \text{estimate CDF of } $w^{(l)} h$
\For{$i = 1, 2, \dots, k$}
    \State $b_i \gets -F^{-1}(i/k)$  \Comment{FC layer biases}
\EndFor

\State \textbf{// Attacker Inference}
\State $\mathcal{H} \gets$ \text{linear layer leakage via Eq.~\eqref{eq: objective}}
\Comment{activations}
\For{$n = 1$ \textbf{to} $|\mathcal{H}|$}
    \State $\tilde{\alpha}_n \gets$ \text{get sparse vector from $h_n$ using Eq.~\eqref{eq: l1}}
    \State $\tilde{x}_n \gets \Psi \, \tilde{\alpha}_n$  \Comment{recover sample from sparse vector}
    \State $\tilde{\mathcal{X}} \gets \tilde{\mathcal{X}} \cup \{ \tilde{x}_n \}$ 
\EndFor

\end{algorithmic}
\end{algorithm}

Our attack for MLP is summarized in Algorithm~\ref{alg:attack_mlp}. Specifically, in the preparation stage, the attacker first initializes the weights of the first FC layer using a Gaussian random distribution (line~2). Based on $W^{(1)}$, the attacker estimates the CDF of the projection values (line~3) and assigns the biases $b_i^{(1)}$ of the first FC layer as the negative quantiles of this distribution (lines~4--6). For the second FC layer, the weight matrix $W^{(2)}$ is initialized with identical rows (line~7), and the corresponding biases $b_i^{(2)}$ are computed similarly using the estimated CDF from $W^{(2)}$ and the activations $h^{(1)}$ (lines~8--10).
In the inference stage, the attacker first reconstructs the set of samples $\tilde{\mathcal{X}}^{(1)}$ from the first FC layer using linear layer leakage (line~13). 
Then it conducts linear layer leakage again on the second FC layer and gets a set of activations $\mathcal{H}^{(1)}$ (line~14). For each activation $h_n^{(1)} \in \mathcal{H}^{(1)}$, a sparse coefficient vector $\tilde{\alpha}_n$ is computed (line~16), and the corresponding input sample $\tilde{x}_n^{(2)}$ is reconstructed by projecting $\tilde{\alpha}_n$ back into the input space using the sparse basis $\Psi$ (line~17). Each recovered sample is appended to the set $\tilde{\mathcal{X}}^{(2)}$, and finally,  $\tilde{\mathcal{X}}^{(1)}$ and $\tilde{\mathcal{X}}^{(2)}$ are combined to obtain the full recovered samples $\tilde{\mathcal{X}}$ (line~20).

\begin{algorithm}[!h]
\caption{ARES for MLP Network}
\label{alg:attack_mlp}
\begin{algorithmic}[1]
\renewcommand{\algorithmicrequire}{\textbf{Input:}}
\renewcommand{\algorithmicensure}{\textbf{Output:}}

\Require Weight and bias gradient $g^{(W)}$, $ g^{(b)}$, sparse basis $\Psi$
\Ensure A set of recovered training samples $\tilde{\mathcal{X}}$

\State \textbf{// Attacker Preparation}
\State $W^{(1)} \gets \text{Gaussian random weight}$  \Comment{first FC layer}
\State $F^{(1)} \gets \text{estimate CDF of } {w^{(1)} x}$
\For{$i = 1, 2, \dots, k^{(1)}$}
    \State $b_i^{(1)} \gets -\bigl(F^{(1)}\bigr)^{-1}\Bigl(\frac{i}{k^{(1)}}\Bigr)$  
    \Comment{first FC layer biases}
\EndFor
\State $W^{(2)} \gets \text{identical rows}$  \Comment{second FC layer}
\State $F^{(2)} \gets \text{estimate CDF of } {w^{(2)} h^{(1)}}$
\For{$i = 1, 2, \dots, k^{(2)}$}
    \State $b_i^{(2)} \gets -\bigl(F^{(2)}\bigr)^{-1}\Bigl(\frac{i}{k^{(2)}}\Bigr)$  
    \Comment{second FC layer biases}
\EndFor

\State \textbf{// Attacker Inference}
\State $\tilde{\mathcal{X}}^{(1)} \gets$ \text{first linear layer leakage} 
\Comment{samples}
\State $\mathcal{H}^{(1)} \gets$ \text{second linear layer leakage}
\Comment{activations}
\For{$n = 1$ \textbf{to} $|\mathcal{H}^{(1)}|$}
    \State $\tilde{\alpha}_{n} \gets$ \text{get sparse vector from $h^{(1)}_n$ using Eq.~\eqref{eq: l1}}
    \State $\tilde{x}^{(2)}_n \gets \Psi \, \tilde{\alpha}_{n}$  \Comment{recover input from sparse vector}
    \State $\tilde{\mathcal{X}}^{(2)} \gets \tilde{\mathcal{X}}^{(2)} \cup \{ \tilde{x}^{(2)}_n \}$ 
\EndFor
\State $\tilde{\mathcal{X}} \gets \tilde{\mathcal{X}}^{(1)} \cup \tilde{\mathcal{X}}^{(2)}$


\end{algorithmic}
\end{algorithm}


In the following, we provide the expected recovery rate for both the one-layer and two-layer recovery.
Intuitively, the expected number of recovered samples can be derived by enumerating all possible assignments of $N$ samples into $k$ equal-mass bins. A sample is considered recovered if it occupies a bin alone, since in this case the linear-layer inversion can uniquely identify that sample.
Formally, the expected recovery rate for a single FC layer is given by
\begin{align}
E(N, k) =\ & \frac{1}{\dbinom{k + N - 1}{k - 1}} \sum_{i=1}^{N - 2} i \binom{k}{i}
\nonumber \\
& \times \sum_{j=1}^{\left\lfloor \frac{N - i}{2} \right\rfloor}
\binom{k - i}{j} \binom{N - i - j - 1}{j - 1}
+ r(N, k),
\label{eq:expected_proportion}
\end{align}
where $k$ is the number of output neurons and $N$ is the number of samples in the batch, and $r(n,k)$ captures special configurations, e.g., samples didn't fall into any bin (cf. RtF \cite{fowlrobbing}).
Let $p_1= E(N, k^{(1)})/{N}$ denote the proportion of samples successfully recovered in the first layer.
The proportion of samples successfully recovered second layer is $p_2= E(N(1-p_1), k^{(2)})/{N(1-p_1)}$. Consequently, the total number of samples recovered across the two layers is $N \left[ p_1 + (1 - p_1) p_2 \right]$.
As illustrated in Fig.~\ref{fig:k-n}, increasing the number of output neurons results in a higher expected recovery rate. Moreover, adding a second FC layer yields a substantial performance improvement; for instance, with 1024 output neurons per layer, a two-layer configuration can recover over 80\% of the samples in a batch of 384.
\begin{figure}
    \centering   \includegraphics[width=\linewidth]{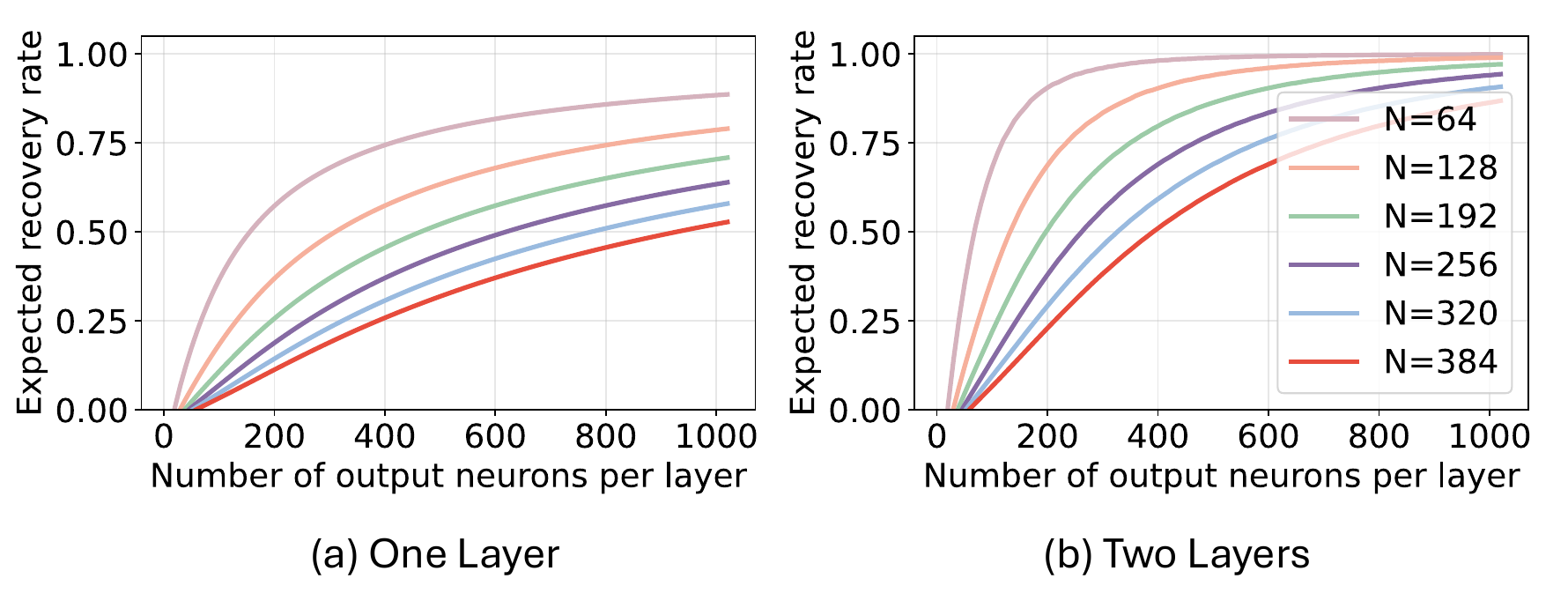}
    \caption{Expected recovery rate with different numbers of output neurons ($k$) and batch size ($N$). }
    \label{fig:k-n}
\end{figure}




\section{Experiments}
We conduct extensive experiments to evaluate the effectiveness of ARES. Section~\ref{sec:effect_comparsion} compares ARES with SOTA attacks on CNN- and MLP-based networks. Section~\ref{sec:effect_defense} evaluates ARES under gradient perturbation, data augmentation, and secure aggregation defenses. Section~\ref{sec:diverse} explores ARES on text and audio data, non-IID scenarios, FedAvg, and asynchronous FL settings.
\subsection{Experimental Setup}
\noindent\textbf{Dataset.}
We use five image datasets, including MNIST \cite{lecun1998mnist}, CIFAR-10 \cite{krizhevsky2009learning}, ImageNet \cite{deng2009imagenet}, HAM10000 \cite{tschandl2018ham10000}, Lung-Colon Cancer \cite{borkowski2019lung}, one text dataset, i.e., Wikitext dataset \cite{merity2016pointer}, and one audio dataset, i.e., AudioMNIST \cite{becker2024audiomnist} to evaluate the effect of our attack.

\noindent\textbf{Evaluated Networks.} 
We adopt two representative network architectures: a 4-layer CNN, which consists of convolutional layers followed by FC layers, and a 4-layer MLP, which comprises four FC layers. 
The detailed descriptions for networks are provided in Table~\ref{table:network} in the Appendix.

\noindent\textbf{Evaluation Metric.} We employ four metrics to assess the effectiveness of our attack, including PSNR (higher is better), MSE (lower is better), LPIPS (lower is better), and recovery rate (higher is better). 
For PSNR, MSE, and LPIPS, we report the averaged value across the batch.
The detailed descriptions for each metric are provided in Appendix~\ref{app: eval}.

\noindent\textbf{Compared Attacks.}
We compare our methods with nine state-of-the-art GIAs, including iDLG \cite{zhao2020idlg}, InvertingGrad (IG) \cite{geiping2020inverting}, GradInversion (GI) \cite{yin2021see}, FedLeak \cite{fan2025boosting}, Fishing \cite{wen2022fishing}, Robbing (RtF) \cite{fowlrobbing}, Trap Weight (TW) \cite{boenisch2023curious}, LOKI \cite{zhao2024loki} and Scale-MIA \cite{shi2023scale}.
To implement RtF and TW in a CNN-based network, we adopt the approach proposed in TW \cite{boenisch2023curious}, which uses the direct-pass initialization method for convolutional layers to avoid value distortion.
Our implementation for each method is based on the open-source code\footnote{\url{https://github.com/lhfowl/robbing_the_fed}, \url{https://github.com/JonasGeiping/breaching}, \url{https://github.com/unknown123489/Scale-MIA}.}.
The detailed descriptions for each attack are provided in Appendix~\ref{app: attacks}.

\noindent\textbf{Evaluated Defenses.}
We consider three categories of defenses. First, we evaluate three gradient perturbation–based approaches, including differential privacy (DP) \cite{geyer2017differentially}, gradient quantization \cite{yue2023gradient}, and gradient sparsification \cite{yue2023gradient}. Second, we consider a data augmentation–based defense, ATS \cite{gao2023automatic}, implemented using public code\footnote{\url{https://github.com/gaow0007/ATSPrivacy}}
. Finally, we examine secure aggregation–based defenses \cite{bonawitz2017practical,fereidooni2021safelearn}. 
The detailed explanations of each defense are provided in Appendix~\ref{app: defenses}.




\subsection{Performance Comparison with Baselines}
\label{sec:effect_comparsion}
\noindent\textbf{Comparison on CNN-based Networks.}
Table~\ref{tab:attack-impact} demonstrates that our attack consistently outperforms SOTA GIAs across all datasets and practical batch sizes. We leave the visual illustration of the recovery effect in Appendix~\ref{app: visual} (Fig.~\ref{fig:visual}).
For fairness, we focus on comparison with RtF and TW, as both are active attacks that leverage the linear layer leakage technique.
\begin{figure}
    \centering   \includegraphics[width=0.9\linewidth]{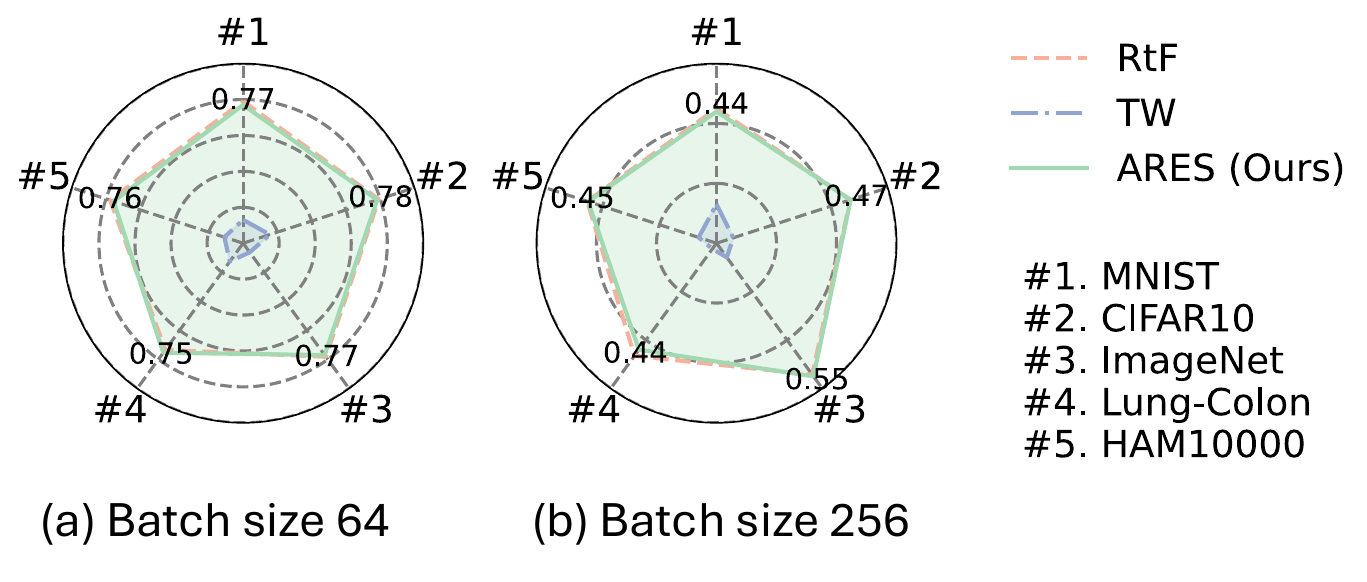}
    \caption{Recovery rate comparison in CNN.}
    \label{fig:recovery_rate_exp}
\end{figure}
As shown in Fig.~\ref{fig:recovery_rate_exp}, both RtF and ARES achieve a higher recovery rate than TW, with improvements of {6.58~$\times$ and 6.57~$\times$ for a batch size of 64, and 7.14~$\times$ and 7.05~$\times$} for a batch size of 256, respectively. This improvement is attributed to the use of the imprint method, which maximizes the likelihood of assigning each sample to distinct bins, thereby enhancing the probability of separating samples within a batch.
However, even at a comparable recovery rate, ARES achieves a significantly higher PSNR than RtF, with an average improvement of 4.8~$\times$.
This improvement stems from our method’s ability to address the key challenge of reconstructing activations back to the original training samples. Although initializing convolutional kernels using the direct-pass method can partially mitigate value distortion, ReLU activations in preceding layers still mask out almost half of the signals. 

We further compare the empirical PSNR for single image with the theoretical bound provided by Theorem~\ref{therm: error}. Theorem~\ref{therm: error} establishes an upper bound on the $\ell_2$ distance between the estimated and ground-truth signals, which translates into a lower bound on PSNR for image reconstruction. Under the ARES design, the theoretical MSE upper bound is around 0.25, corresponding to a PSNR lower bound of 54. In practice, the PSNR for a single image can exceed 100, which remains well within the guarantee provided by the theorem. This gap reflects the conservative nature of the bound, as it holds for any $s$-sparse vector, including worst-case or adversarially constructed signals, whereas real images typically exhibit smooth regions and structured patterns that facilitate easier reconstruction.

\begin{table*}[ht]
\centering
  \caption{PSNR comparison of state-of-the-art attacks versus our ARES across different batch sizes and datasets. For each row, the best-performing attack is highlighted in bold. \xmark~indicates that the attack fails to get visually meaningful images.}
  \label{tab:attack-impact}
\setlength{\tabcolsep}{6pt}
\renewcommand{\arraystretch}{1.2}
\begin{tabular}{c |c|c c c c c c c c}
\toprule
Dataset & Batch & iDLG & IG  & GI  & FedLeak  & Finshing$^*$  & RtF  & TW  & ARES (Ours)\\
\midrule
\multirow{3}{*}{MNIST} 
 & 32  & 9.39   & 9.96 & 11.17 & 21.57  & 12.40  & 17.53  &13.89 & \textbf{105.29}  \\
 & 64  &  9.27   & 9.88 & 10.94 & 21.12   &  12.58 &  16.68 & 13.02& \textbf{94.16}  \\
 & 256 & \xmark    & \xmark    & \xmark    & \xmark            &  12.40& 13.08 & 11.87 & \textbf{42.42}  \\
\midrule
\multirow{3}{*}{CIFAR-10} 
 & 32  & 9.51 & 11.28 & 10.59 & 21.33   &12.08 & 17.35 & 13.55 & \textbf{92.27} \\
 & 64  & 8.52 & \xmark  & 10.01 & 17.20  & 12.36& 17.20 &  13.03     & \textbf{90.47} \\
 & 256 & \xmark    & \xmark  & \xmark     & \xmark      & 12.36 & 15.08 &    13.51   & \textbf{37.08} \\
\midrule
\multirow{3}{*}{ImageNet} 
 & 32  & \xmark    & 11.45 & 8.06  & 19.07 & 12.71&  16.58 & 12.92 & \textbf{104.89} \\
 & 64  & \xmark    & 10.75 & \xmark           & 19.01 & 12.68&  15.20 & 12.59 & \textbf{90.72}  \\
 & 256 & \xmark    & \xmark  & \xmark           & \xmark     & 12.49&  14.86 & 12.56 & \textbf{41.61}  \\
\midrule
\multirow{3}{*}{HAM10000} 
 & 32  & \xmark & \xmark & \xmark & 15.32 &12.67   & 15.32 & 15.29  & \textbf{120.93} \\
 & 64  & \xmark & \xmark & \xmark & 22.16 & 12.52  & 15.01  & 14.39  & \textbf{69.96}  \\
 & 256 & \xmark & \xmark & \xmark & \xmark     & 12.48  & 14.97  & 14.88  & \textbf{37.72} \\
\midrule
\multirow{3}{*}{Lung-Colon} 
 & 32  & \xmark & \xmark & \xmark & 17.73 & 12.88  & 17.10 & 14.36  & \textbf{106.64} \\
 & 64  & \xmark & \xmark & \xmark & 16.06 & 12.41  & 16.64  & 14.20  & \textbf{67.37}  \\
 & 256 & \xmark & \xmark & \xmark & \xmark     & 12.27  & 15.06  & 13.84 & \textbf{33.12}  \\
\bottomrule
\end{tabular}
\begin{tablenotes}
\item * Fishing column reports the PSNR for a single sample; others present the average PSNR across all samples in the batch.
\end{tablenotes}
\end{table*}

\noindent\textbf{Comparison on MLP-based Networks.}
\begin{figure}
    \centering   \includegraphics[width=0.9\linewidth]{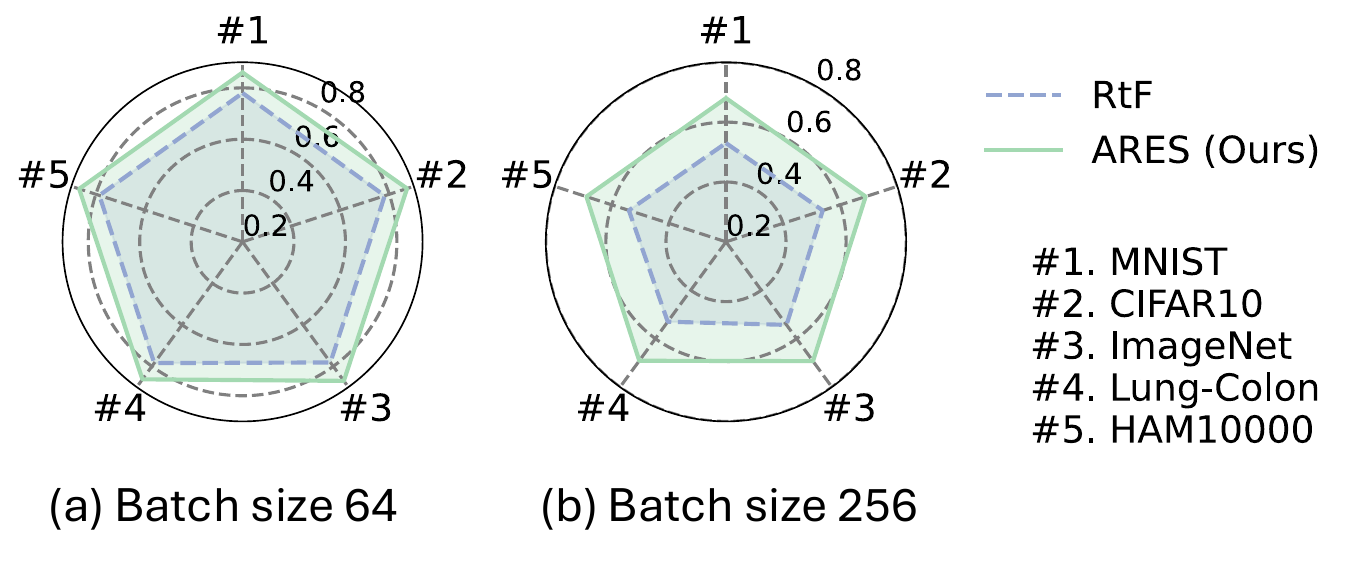}
    \caption{Recovery rate comparison in MLP. }
    \label{fig:recovery_rate_exp_mlp}
\end{figure}
To evaluate our attack on an MLP-based network, we primarily compare our method with RtF, which is currently the strongest active attack. As shown in Fig.~\ref{fig:recovery_rate_exp_mlp}, ARES improves the recovery rate by 10.75\% and 28.34\% for batch sizes of 64 and 256, respectively.
This improvement arises from our method’s ability to perform second-layer separation.

\noindent\textbf{Theoretical and Empirical Recovery Rates.}
We compare the expected recovery rates derived from Eq.~\eqref{eq:expected_proportion} with the empirical recovery rates obtained by averaging the results across the datasets reported in Table~\ref{tab:attack-impact}. As shown in Fig.~\ref{fig:expected&experiment_recovery_rate}, the empirical results closely follow the expected values across different settings, with an average deviation of 3.5\%.

\begin{figure}
    \centering   \includegraphics[width=\linewidth]{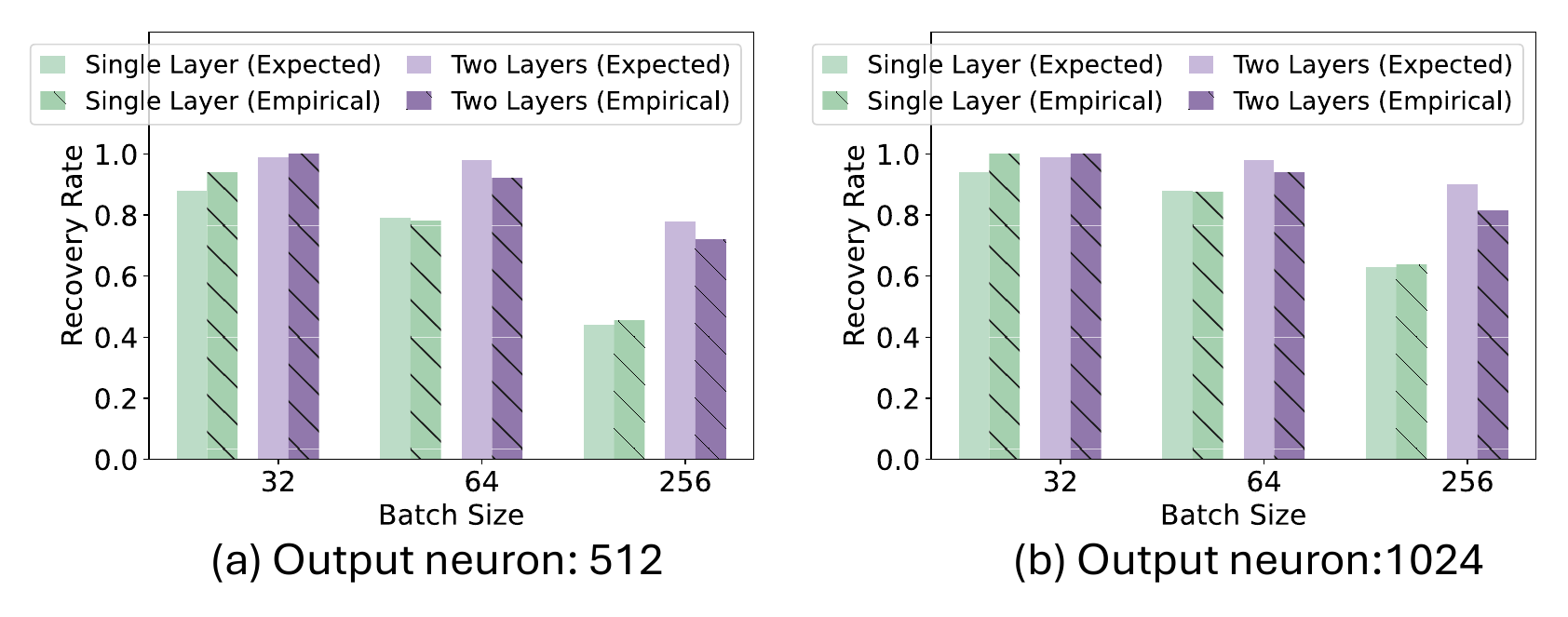}
    \caption{Expected vs. empirical recovery rates.}
    \label{fig:expected&experiment_recovery_rate}
\end{figure}

\subsection{Performance Under Different Defenses}
\label{sec:effect_defense}
\noindent\textbf{Attack Against Gradient Perturbation-based Defense.}
Fig.~\ref{fig:defense_quant} shows the performance of ARES under gradient quantization defense with a batch size of 32 on the ImageNet dataset.
For fairness, we compare only with RtF, which is the strongest baseline according to Table~\ref{tab:attack-impact}.
As shown, ARES consistently outperforms RtF on both CNN and MLP networks, achieving PSNR improvements of 5× and 1.35×, respectively.
As the quantization bit increases, the reconstruction quality improves. 
Fig.~\ref{fig:defense-sparse} shows the performance of ARES under gradient sparsification defense with a batch size of 32 on the ImageNet dataset. Here, the density denotes the fraction of gradient retained after sparsification. 
ARES consistently outperforms RtF on both CNN and MLP networks, achieving PSNR improvements of 3~$\times$ and 1.12~$\times$, respectively.
As this density increases, the reconstruction quality improves.
Fig.~\ref{fig:defense-dp} shows the performance of ARES under different DP noise with a batch size of 32 on the ImageNet dataset.
Following prior works~\cite{fowlrobbing,zhao2024loki}, we apply the Laplace mechanism with varying privacy budgets $\varepsilon$. ARES consistently outperforms RtF on both CNN and MLP networks, achieving PSNR improvements of 2.5~$\times$ and 1.14~$\times$, respectively. At $\varepsilon = 10$, the recovered images achieve PSNR $\approx 20$. Further reducing the privacy budget leads to a significant drop in training accuracy.
We leave the visual effect in Appendix~\ref{app: visual} (Fig.~\ref{fig:visual_quant}, Fig.~\ref{fig:visual_sparse} and Fig.~\ref{fig:visual_dp}).
\begin{figure}
    \centering   \includegraphics[width=\linewidth]{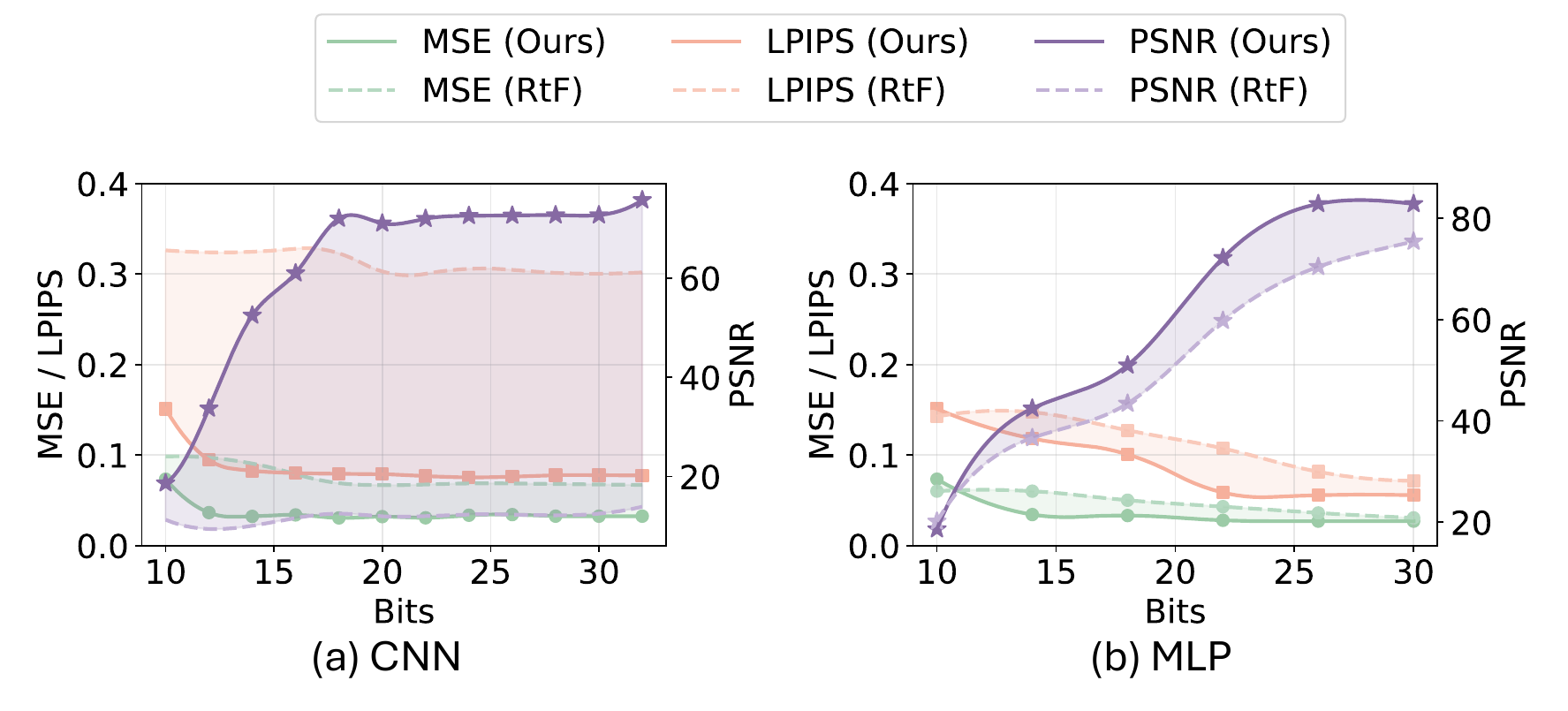}
    \caption{Our attack under gradient quantization defense.}
    \label{fig:defense_quant}
\end{figure}

\begin{figure}
    \centering   \includegraphics[width=\linewidth]{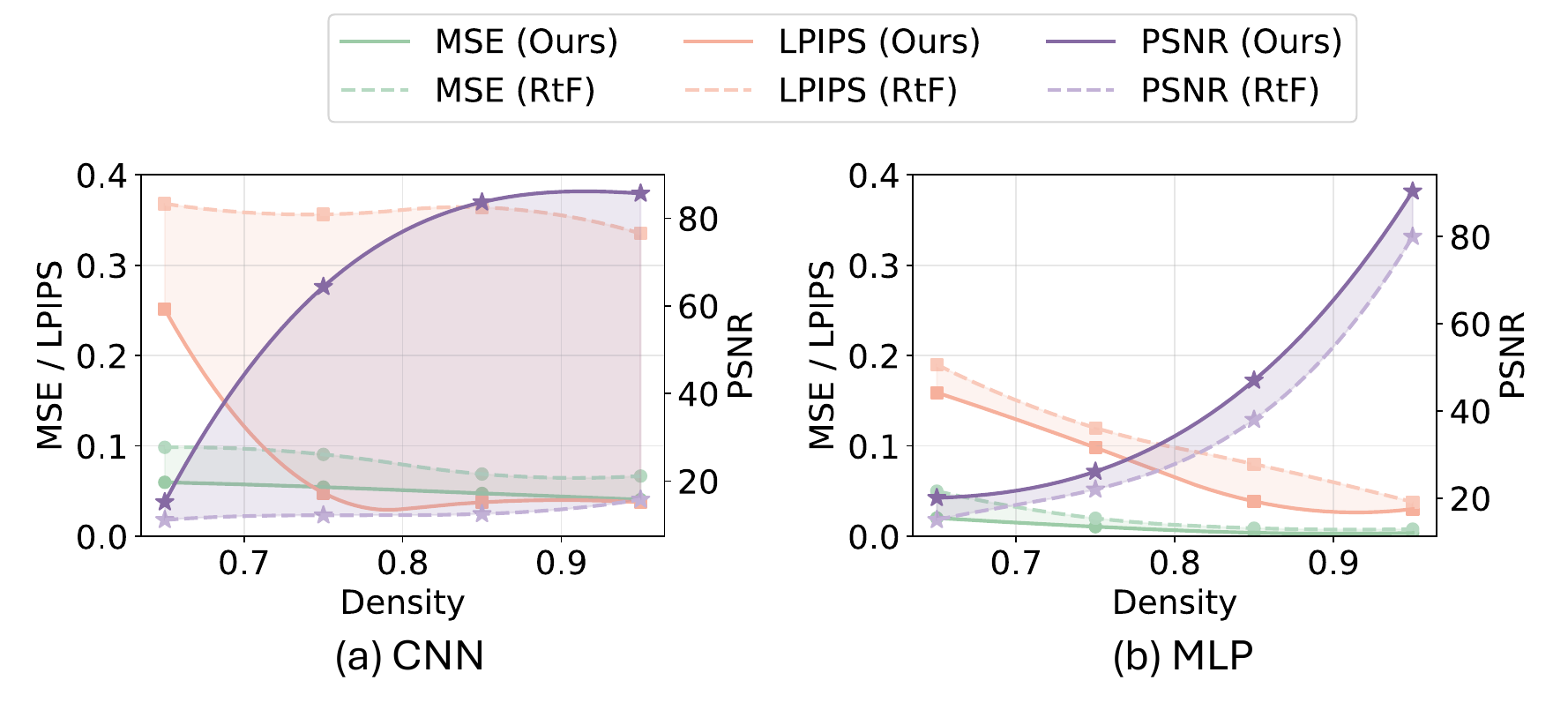}
    \caption{Our attack under gradient sparsification defense.}
    \label{fig:defense-sparse}
\end{figure}
\begin{figure}
    \centering   \includegraphics[width=\linewidth]{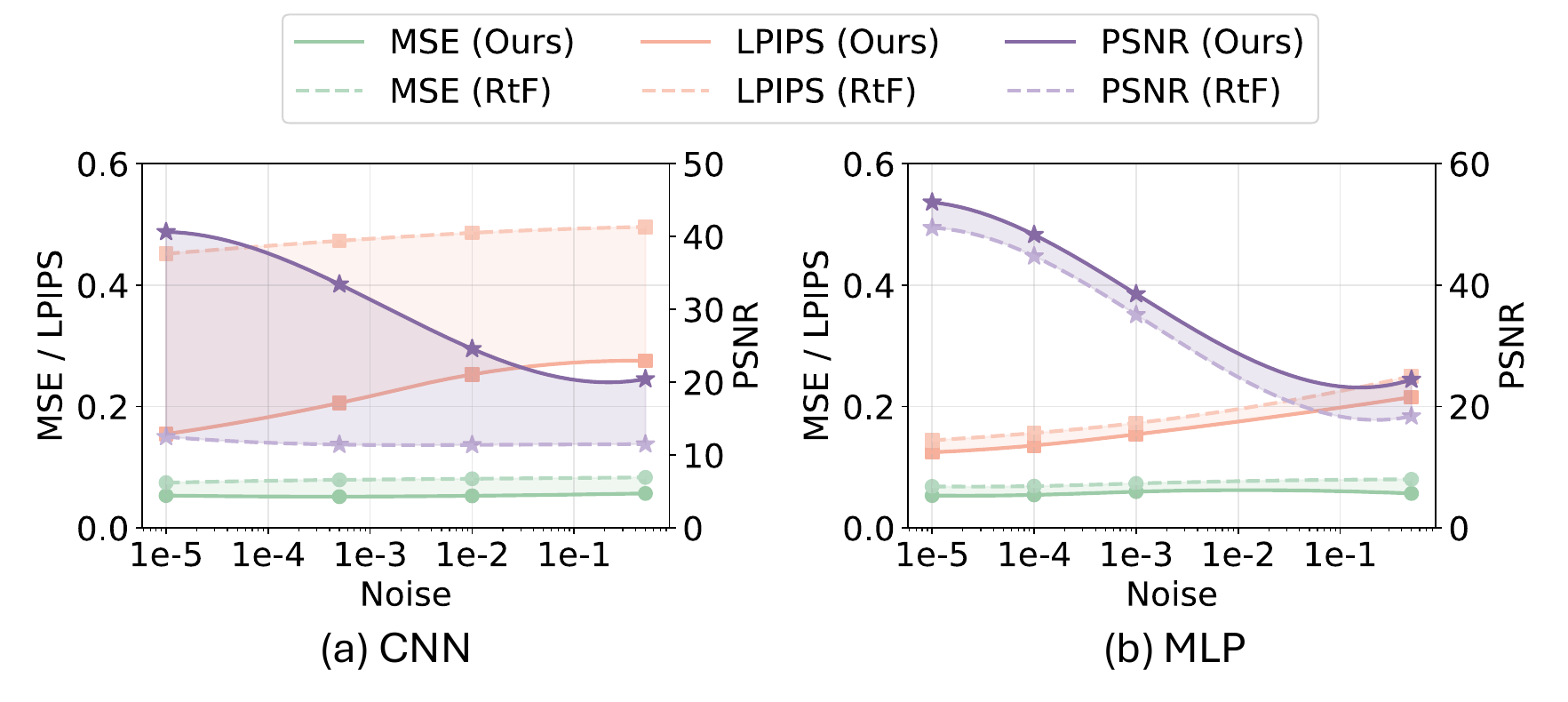}
    \caption{Our attack under differential privacy noise.}
    \label{fig:defense-dp}
\end{figure}

\noindent\textbf{Attack Against Data Augmentation-based Defense.}
We evaluate our ARES against ATS \cite{gao2023automatic}, a data augmentation–based defense that aims to learn the optimal augmentation policy to prevent the reconstruction of both original and transformed training samples, while preserving model utility. 
Our attack achieves an average PSNR of 86.8 (comparing reconstructed samples with the transformed samples) on the CIFAR-10 dataset with a batch size of 32. 
We leave the visual illustration of the augmented samples and the recovered samples in Appendix~\ref{app: visual} (Fig.~\ref{fig:visual_aug}).

\noindent\textbf{Attack Against Secure Aggregation Defense.}
A commonly used strategy to bypass secure aggregation is to exploit model inconsistency \cite{pasquini2022eluding,zhao2024loki}, where the server sends different model parameters (within the same model architecture) to each client, preventing the mixing of gradients from different clients.
One such approach is LOKI \cite{zhao2024loki}, which leverages model inconsistency to perform GIA under secure aggregation defense. To evaluate our method under the same defense, we adopt LOKI’s setup and deliberately manipulate convolutional weights to introduce model inconsistency, preventing the clients’ weight gradients mixing during secure aggregation.
Because ARES is orthogonal to LOKI (i.e., model inconsistency method), integrating ARES into the LOKI setup yields complementary attack capabilities.
In our experiment, each client receives a model in which only a client-unique subset of kernels (e.g., three per client) is initialized with the malicious weight (lines 2--3 in Algorithm 1), while all remaining kernels are set to zero. This design keeps each client’s weight gradients in the FC layer unmixed thus improve the recovery rate.
We evaluate this approach using 10 clients per training round.
As shown in Fig.~\ref{fig:se-agg}, our method achieves an average PSNR of 53.7 and 51.9 on CNN-based networks with global batch sizes (i.e., local batch size $\times$ number of clients per round) of 320 and 640, respectively, outperforming LOKI by 7.3~$\times$ and 7.1~$\times$.
The higher PSNR is achieved by addressing the challenge of reconstructing training samples from activations with minimal information loss.
Although LOKI can initialize the convolutional kernels using the direct-pass method to reduce value distortion, ReLU activations still suppress half of the inputs to the FC layer, resulting in information loss.

\begin{figure}
    \centering   \includegraphics[width=0.9\linewidth]{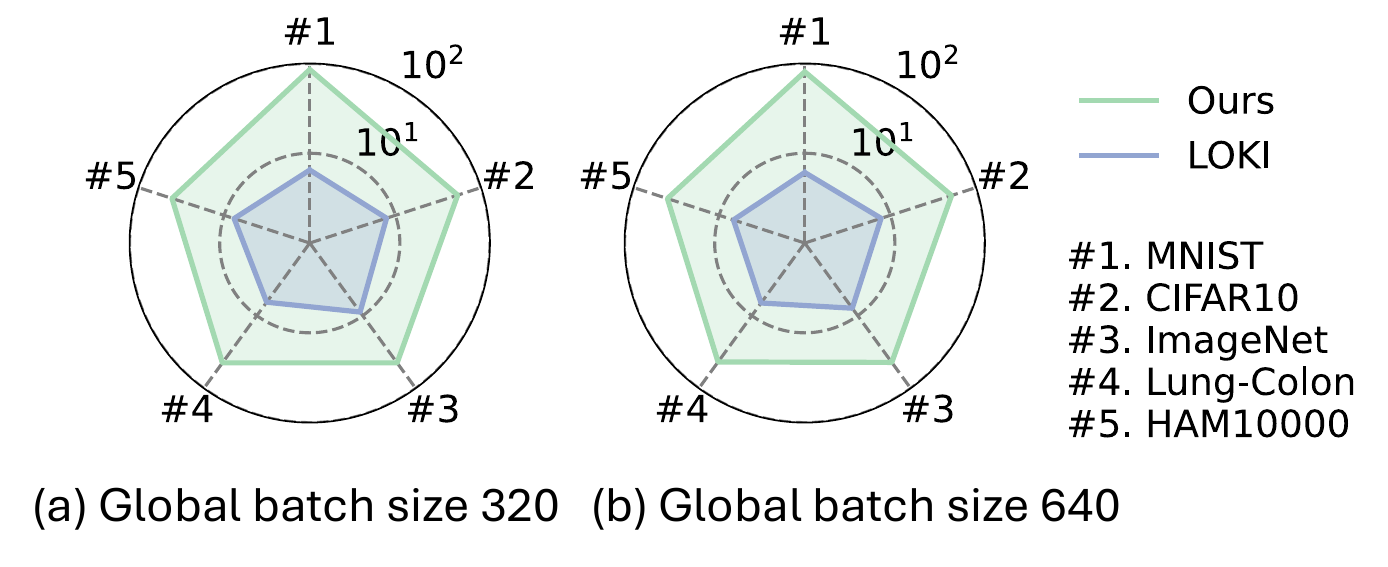}
    \caption{PSNR of LOKI and ours under secure aggregation-based defenses. The radial axis is log-scaled for better visualization.}
    \label{fig:se-agg}
\end{figure}
\begin{table}[h]
\centering
\begin{tabular}{c|cccc}
\toprule
    Global batch size & 32 & 64 & 128 & 256 \\
\midrule
Scale-MIA & 28.52 & 28.53 & 28.26 & 27.06 \\
ARES & 92.27 & 90.47 & 85.01 & 37.08 \\
\bottomrule
\end{tabular}
\caption{PSNR comparison of Scale-MIA and ARES on CIFAR-10 dataset in CNN.}
\label{tab:scale-mia}
\end{table}
We also compare ARES with an attack that relies on stronger adversarial assumptions, namely Scale-MIA \cite{shi2023scale}, which assumes the attacker has access to a subset of the training data to train a decoder that maps activations back to the original samples. As shown in Table~\ref{tab:scale-mia}, ARES achieves an average PSNR improvement of 2.71~$\times$ over Scale‑MIA despite having less prior knowledge.


\subsection{Attack in Diverse Settings}
\label{sec:diverse}
\noindent\textbf{Attack on Text Data.}
We use an MLP network to test the effect of our attack.
We report three evaluation metrics commonly used in text datasets, including accuracy, BLEU score, and ROUGE-L (details in Appendix~\ref{app: eval}).
As shown in Fig.~\ref{fig:wikitext}, our method consistently outperforms RtF across all combinations of sequence lengths and batch sizes, achieving improvements of 5.76\%, 5.84\%, and 8.92\% in accuracy, BLEU, and ROUGE-L on batch size 32, and 15.71\%, 2.91\%, and 13.29\% on batch size 64.
We leave the recovery effect in Appendix~\ref{app: visual} (Fig.~\ref{fig:text_visual}).
\begin{figure}
    \centering   \includegraphics[width=\linewidth]{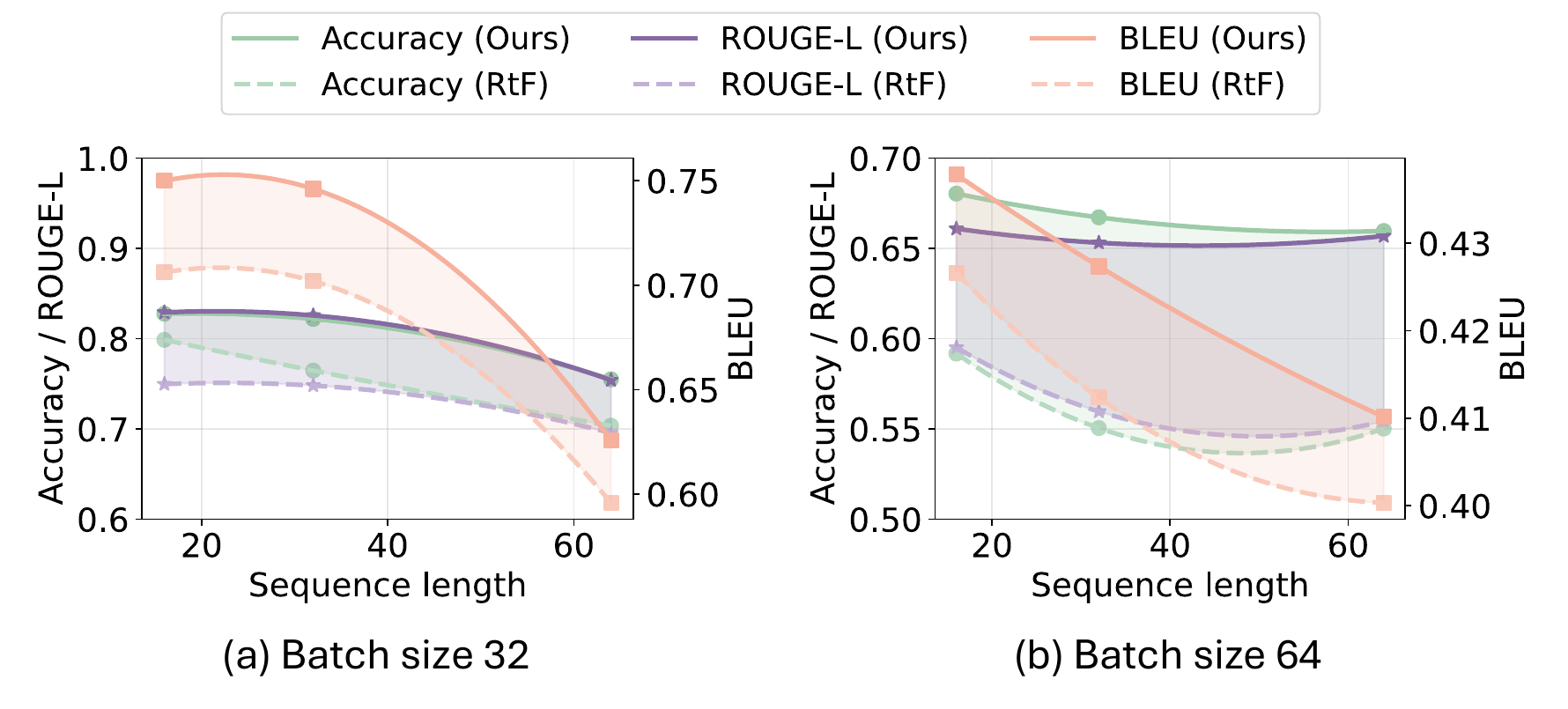}
    \caption{Our attack on the Wikitext dataset.}
    \label{fig:wikitext}
\end{figure}

\noindent\textbf{Attack on Audio Data.}
We use a CNN to test the effect of our attack.
Our method achieves an average PSNR of 58.55 and 45.83 on a batch size of 32 and 64, respectively. 
Ground truth and recovered audio files (WAV) are available\footnote{\url{https://github.com/gongzir1/ARES/tree/main/AudioExample}}.



\noindent\textbf{Label Skew.}
\begin{figure}
    \centering   \includegraphics[width=\linewidth]{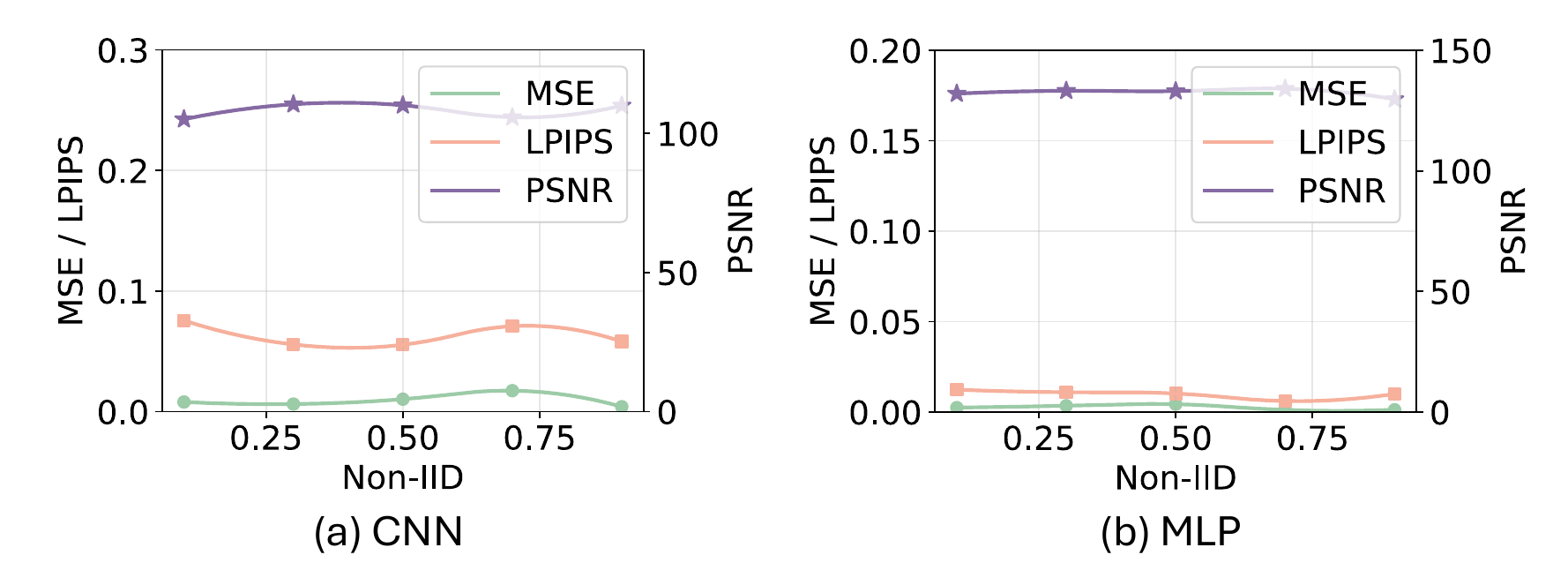}
    \caption{Our attack on non-iid data with label skew.}
    \label{fig: non-iid-label}
\end{figure}
We evaluate our attack under varying degrees of label skew, where each client holds only a subset of classes. The skew scalar controls the level of non-IIDness (0: IID; larger values indicate fewer classes per client). Results on the CIFAR-10 dataset with a batch size of 32 show that our attack consistently achieves strong performance across all levels of label skew (Fig.~\ref{fig: non-iid-label}).

\noindent\textbf{Feature Skew.}
We evaluate our attack under varying degrees of feature skew by partitioning the dataset in the feature space. Specifically, we first project samples onto the top principal components using Principal Component Analysis (PCA) and divide the resulting feature space into multiple regions. Each client is then assigned samples from only a subset of these regions, creating feature distribution differences across clients. The skew scalar controls the number of regions assigned to each client (0: IID; larger values correspond to fewer regions per client and therefore stronger feature skew).
Experimental results on the CIFAR-10 dataset with a batch size of 32 show that, in the CNN network, the PSNR decreases as feature skew increases because more samples fall into the same bin compared to the standard IID case. Nevertheless, the attack still achieves a PSNR of around 100 dB. In contrast, in the MLP network, our attack performs well across all degrees of feature skew, as the second layer enables further separation of samples (Fig.~\ref{fig: non-iid-feature}).

\begin{figure}
    \centering   \includegraphics[width=\linewidth]{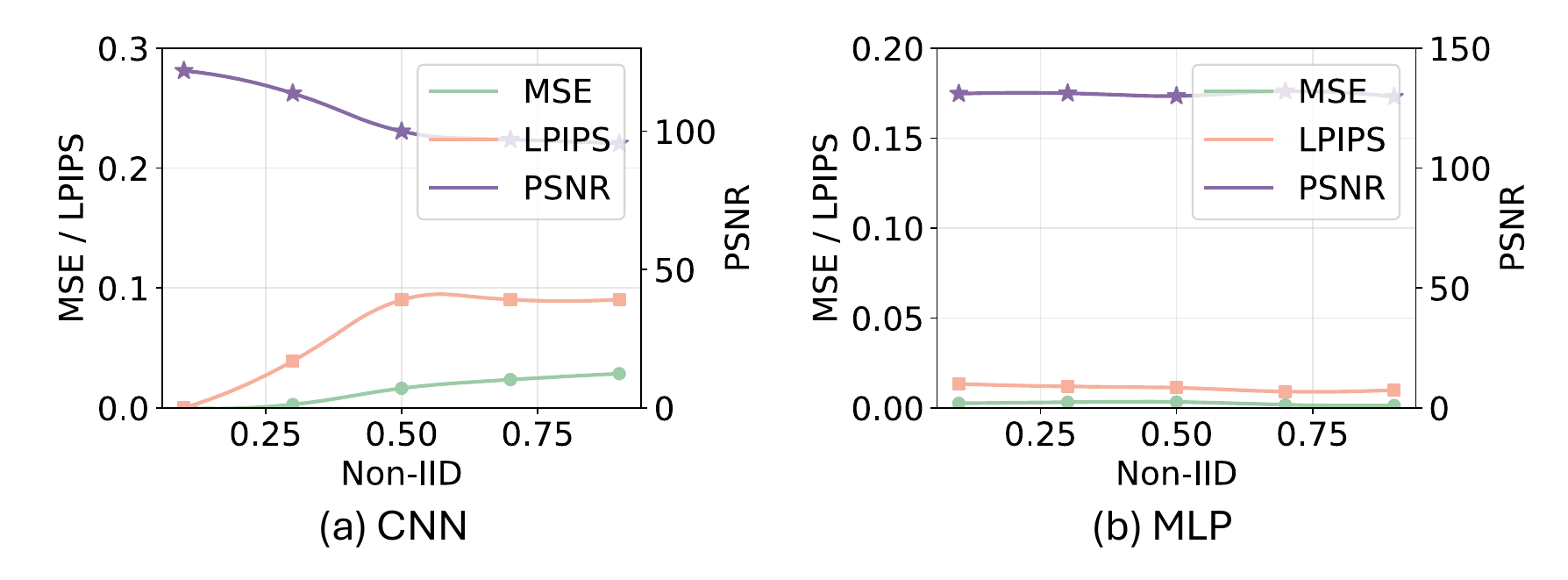}
    \caption{Our attack on non-iid data with feature skew.}
    \label{fig: non-iid-feature}
\end{figure}



\noindent\textbf{Attack on FedAvg.}
\begin{figure}
    \centering   \includegraphics[width=\linewidth]{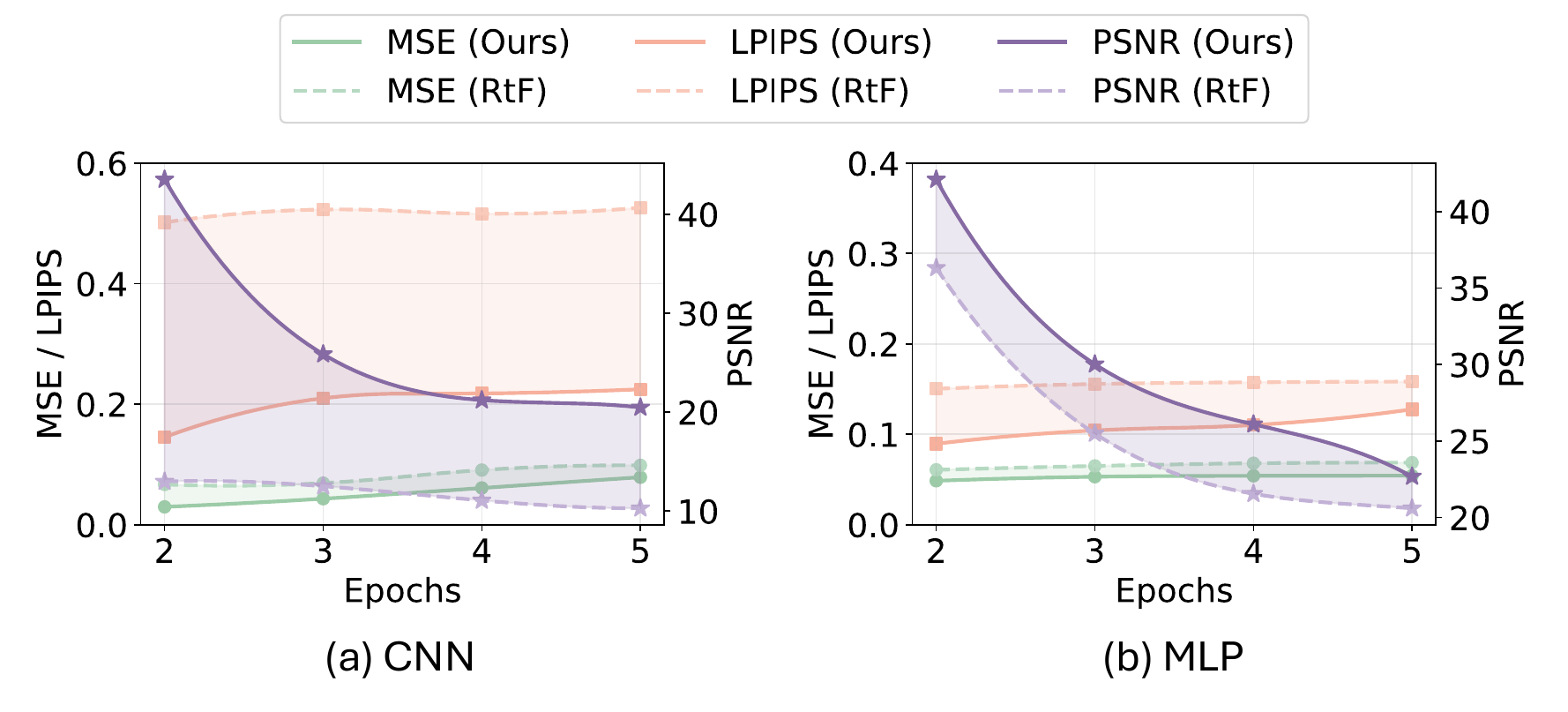}
    \caption{Our attack in the FedAvg setting.}
    \label{fig:fedavg}
\end{figure}
In FedAvg, each client trains locally on its own data for $T$ epochs and then sends model updates to the server.
To attack under this setting, the attacker first estimates the gradient using
\begin{equation}
\label{eq: fedavg_estimate_gradient}
\hat{g} \approx - \frac{1}{lr \, T} \Delta W,
\end{equation}
where $\hat{g}$ is the estimated gradient, $lr$ is the local learning rate, $T$ is the number of local training epochs, $\Delta W$ is the client’s model update.
Next, the server further estimates the intermediate weight for each local epoch as
\begin{equation}
\label{eq: fedavg_estimate_w}
    \hat{W}_t=W_g- t \, lr \hat{g}, \quad t = 1, \dots, T,
\end{equation}
where $\hat{W}_t$ is the estimated model at local epoch $t$ and $W_g$ is the global model at the start of the round.
After obtaining the estimated gradient and weight, the attacker uses Eq. \eqref{eq: objective} and Eq. \eqref{eq: l1} to get the training samples. 
We evaluate our method on FedAvg using the HAM dataset with a local batch size of 8.  
As shown in Fig.~\ref{fig:fedavg}, our approach consistently outperforms RtF\footnote{We compare with the main method (Eq. (4) in \cite{fowlrobbing}). Although the original paper proposes a variant that performs well under the FedAvg setting, it requires a malicious modification of the activation function, which is beyond the scope of our threat model.} on both CNN and MLP networks, achieving PSNR improvements o 2~$\times$ and 1.16~$\times$, respectively.

\noindent\textbf{Attack on Asynchronous FL.} Similarly, we use Eq. \eqref{eq: fedavg_estimate_gradient} and Eq. \eqref{eq: fedavg_estimate_w} to estimate the gradient and local model weight for each client. And use Eq. \eqref{eq: objective} and Eq. \eqref{eq: l1} to get the training samples. 
We evaluate our method on the Asynchronous FL framework \cite{nguyen2022federatedasynchronous} using the HAM dataset with a local batch size of 8. As shown in Fig.~\ref{fig:asynchronous}, our approach consistently outperforms RtF on both CNN and MLP networks, achieving  PSNR improvements of 2~$\times$ and 1.15~$\times$ in PSNR, respectively.
\begin{figure}
    \centering   \includegraphics[width=\linewidth]{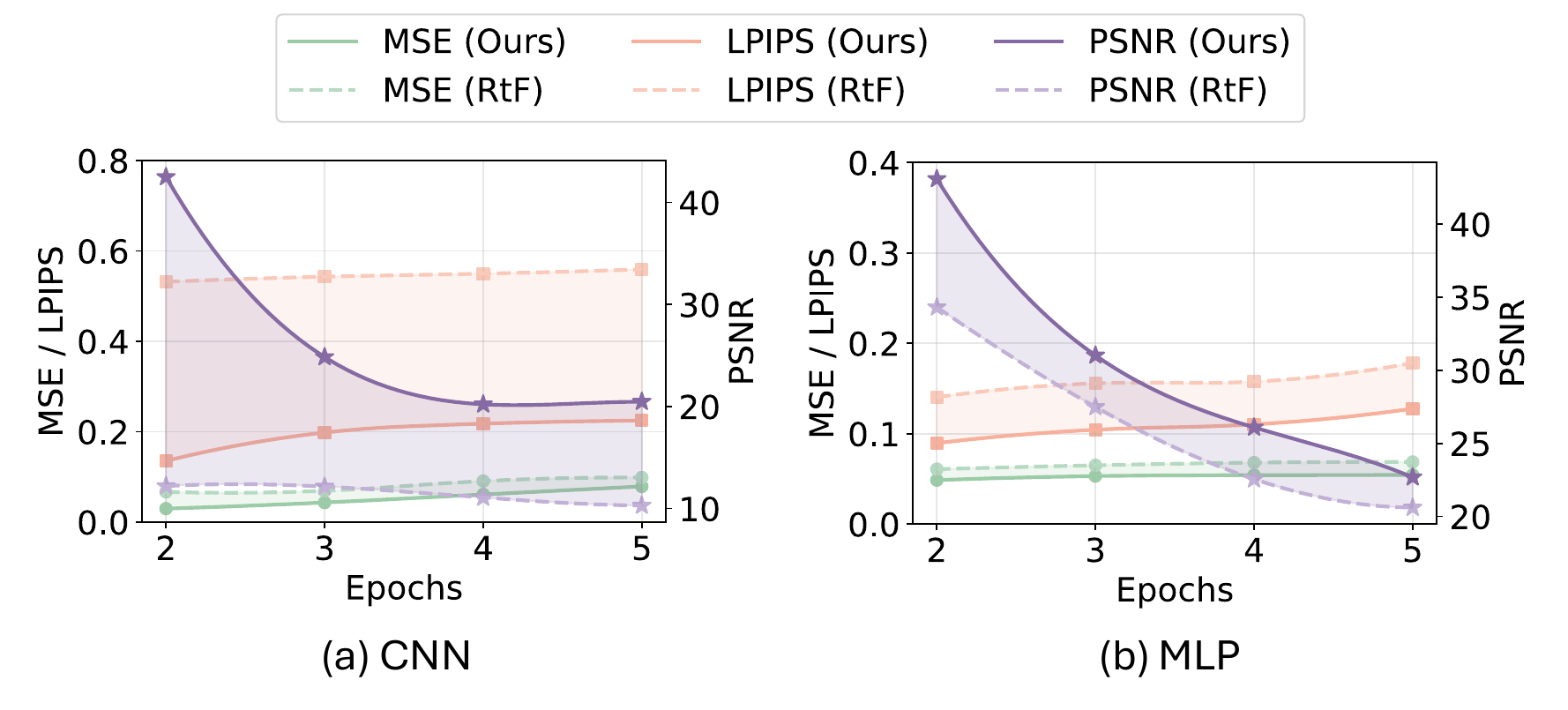}
    \caption{Our attack in the Asynchronous FL setting.}
    \label{fig:asynchronous}
\end{figure}

\noindent\textbf{Impact of Activation Function on Noisy Sparse Recovery.}
We evaluate the performance of our noisy sparse recovery under commonly used activation functions, including GELU, ELU, and SiLU. The problem is solved by optimizing Eq.~(12) using a gradient-based method (SGD). Experiments on the Conv4 network show strong reconstruction performance across all activation functions, achieving single-sample PSNR values of 82.25 dB (GELU), 78.39 dB (ELU), and 85.71 dB (SiLU).

\noindent\textbf{Impact of Activation Functions on Linear Layer Leakage.}
We evaluate the performance of linear-layer leakage under commonly used activation functions, including GELU, ELU, and SiLU. Although the choice of activation function can influence leakage, these activations can be adapted to exhibit ReLU-like behavior. Specifically, ELU reduces to ReLU when its scaling parameter is set to zero. For GELU and SiLU, scaling the pre-activation FC weights pushes activations away from zero, thereby making them behave similarly to ReLU.
Experiments on Conv4 with a batch size of 32 demonstrate strong performance across all activation functions, achieving a PSNR of 100.95 dB with a 97\% recovery rate for ELU, 94.70 dB with a 93\% recovery rate for GELU, and 94.56 dB with a 93\% recovery rate for SiLU.

\section{Conclusion}
In this work, we introduce ARES, a practical and effective active GIA capable of recovering training data from gradients with high fidelity under realistic large batch sizes and reasonable adversary assumptions.
It achieves this by addressing the fundamental challenge of inverting hidden activations into training samples. By exploiting the sparse nature of real-world data and leveraging principles from compressed sensing, we formulate the inversion task as a noisy sparse recovery problem, overcoming the underdetermined and nonlinear challenges inherent in the task.
We establish a theoretical upper bound on the recovery error and validate our approach through extensive experiments across diverse datasets, architectures, and defense mechanisms. The results demonstrate that our ARES consistently achieves superior recovery fidelity, highlighting the underestimated privacy risks of FL in real-world settings. 
Investigating the attack’s scalability to architectures without the Linear+ReLU structure and its robustness against targeted defenses and homomorphic encryption remains a promising direction for future work.
\section{Ethics Considerations}
The experiments are conducted exclusively on public datasets and open-source models within controlled environments, without access to real-world deployments or personal data. No undisclosed vulnerabilities are introduced or exploited, and no human subjects are involved in this study.
\section{LLM Usage Considerations}
LLMs were used for editorial purposes in this manuscript, and all outputs were inspected by the authors to ensure accuracy and originality.
\bibliographystyle{IEEEtran}
\bibliography{main}
\begin{appendices}
\section{Derivations}

\subsection{Activation Recovery from FC Layer}
\label{app: batch_activation_recovery}
For a batch of $N$ samples $\mathcal{N} = \{1, \dots, N\}$, the forward pass of neuron $i$ in layer $l$ for sample $n$ is
\begin{equation}
\label{eq:proof_forward}
    z_{i,n}^{(l)} = W_i^{(l)} h^{(l-1)}_n + b_i^{(l)}, 
\end{equation}
where $h^{(l-1)}_n$ denotes the input vector of sample $n$ to layer $l$.
According to the chain rule, the gradient of the loss with respect to the weight for neuron $i$ is
\begin{equation}
\label{eq:proof_weight_gradient}
    g_i^{(W)}=\frac{\partial \mathcal{L}}{\partial W_i^{(l)}}
    = \sum_{n \in \mathcal{N}} 
    \frac{\partial \mathcal{L}}{\partial z_{i,n}^{(l)}} 
    \frac{\partial z_{i,n}^{(l)}}{\partial W_i^{(l)}}.
\end{equation}
The gradient with respect to the bias $b_i^{(l)}$ is
\begin{equation}
\label{eq:proof_bias_gradient}
    g_i^{(b)}=\frac{\partial \mathcal{L}}{\partial b_i^{(l)}}
    = \sum_{n \in \mathcal{N}}
    \frac{\partial \mathcal{L}}{\partial z_{i,n}^{(l)}} 
    \frac{\partial z_{i,n}^{(l)}}{\partial b_i^{(l)}}.
\end{equation}
Using the linear relationship in Eq. \eqref{eq:proof_forward}, we have
\begin{equation}
\label{eq:proof_linear}
    \frac{\partial z_{i,n}^{(l)}}{\partial W_i^{(l)}} = h^{(l-1)}_n, 
    \quad
    \frac{\partial z_{i,n}^{(l)}}{\partial b_i^{(l)}} = 1.
\end{equation}
Combine Eq.~\eqref{eq:proof_weight_gradient}, Eq.~\eqref{eq:proof_bias_gradient} and Eq.~\eqref{eq:proof_linear}, we have
\begin{align}
    g_i^{(W)} &= \sum_{n \in \mathcal{N}} 
    \frac{\partial \mathcal{L}}{\partial z_{i,n}^{(l)}} h^{(l-1)}_n 
    = \sum_{n \in \mathcal{N}} \gamma_i^n h^{(l-1)}_n, \\
    g_i^{(b)} &= \sum_{n \in \mathcal{N}} 
    \frac{\partial \mathcal{L}}{\partial z_{i,n}^{(l)}} 
    = \sum_{n \in \mathcal{N}} \gamma_i^n,
\end{align}
where $\gamma_i^n := \frac{\partial \mathcal{L}}{\partial z_{i,n}^{(l)}}$ denotes the gradient of the loss with respect to the pre-activation of neuron $i$.
Taking the ratio of the weight and bias gradients yields
\begin{equation}
\label{eq: linear_recovery_batch_appendix}
    \frac{g_i^{(W)}}{g_i^{(b)}} 
    = \frac{\sum_{n \in \mathcal{N}} \gamma_i^n h^{(l-1)}_n}
           {\sum_{n \in \mathcal{N}} \gamma_i^n}.
\end{equation}
This expression corresponds to a \emph{weighted average} of the input vectors 
$\{h^{(l-1)}_n\}_{n=1}^N$, where the weights are given by the gradient coefficients 
$\{\gamma_i^n\}_{n=1}^N$.
When the batch size reduces to $N=1$, Eq.~\eqref{eq: linear_recovery_batch_appendix} simplifies to
\begin{equation}
    \frac{g_i^{(W)}}{g_i^{(b)}} = h^{(l-1)},
\end{equation}
which recovers the single-sample result in Eq.~\eqref{eq: single_fc_recovery}.

\subsection{Derivation of Equal Backpropagation Signals}
\label{app: identical_columns}
Consider neuron $i$ in layer $l$ with backpropagation signal
\begin{equation}
    \gamma_i := \frac{\partial \mathcal{L}}{\partial z_i^{(l)}}
    = \sum_j \frac{\partial \mathcal{L}}{\partial z_j^{(l+1)}} 
      \frac{\partial z_j^{(l+1)}}{\partial z_i^{(l)}}
    = \sum_j \frac{\partial \mathcal{L}}{\partial z_j^{(l+1)}} \, W_{ji}^{(l+1)},
\end{equation}
where $W^{(l+1)}$ is the weight matrix of layer $l+1$.
Similarly, for neuron $k$, 
\begin{equation}
    \gamma_{k} = \sum_j \frac{\partial \mathcal{L}}{\partial z_j^{(l+1)}} \, W_{jk}^{(l+1)}.
\end{equation}
Once  
\begin{equation}
    W_{ji}^{(l+1)} = W_{j,k}^{(l+1)},
\end{equation}
then it can achieve $\gamma_i = \gamma_{k}$.

\section{Detailed Explanation on Experiment Set Up}

\begin{table}[h]
\centering
\begin{tabular}{c|c}
\toprule
\textbf{Architecture} & \textbf{Layers} \\
\midrule
&Conv(out=12, k=3, s=1, p=1, act=relu) \\
CNN &Conv(out=24, k=3, s=1, p=1, act=relu) \\
&FC(k=1024, act=relu) \\
&FC(k=\#class, act=softmax) \\
\midrule
 & FC(k=512, act=relu), \\
MLP&FC(k=512, act=relu), \\
&FC(k=512, act=relu), \\
&FC(k=\#class, act=softmax) \\
\bottomrule
\end{tabular}
\caption{Network architectures. Conv: \textit{out} = number of filters, \textit{k} = kernel size, \textit{s} = stride, \textit{p} = padding, \textit{act} = activation. FC: \textit{k} = number of neurons, \textit{act} = activation.
}
\label{table:network}
\end{table}

\subsection{Evaluation Metrics}
\label{app: eval}
\noindent\textbf{MSE} (Mean Squared Error) is a metric that computes the average of squared intensity differences between the reconstructed image and the ground truth. 
Lower MSE indicates better pixel-wise fidelity.
\noindent\textbf{PSNR} (Peak Signal-to-Noise Ratio) is a metric that quantifies the fidelity of reconstructed images relative to ground truth. It is defined as a logarithmic function of the MSE between two images, with higher values indicating better reconstruction quality. 
\noindent\textbf{LPIPS} (Learned Perceptual Image Patch Similarity) is a metric that measures perceptual distance using deep neural network features pretrained on large-scale image data. Lower values indicate higher perceptual similarity. 
\noindent\textbf{Recovery Rate.} We count the number of samples that fall into distinct bins and divide this number by the batch size to compute the recovery rate.
\noindent\textbf{Reconstruction accuracy} measures the proportion of tokens in the reconstructed text that match the original tokens. 
\noindent\textbf{BLEU score} evaluates the overlap of n-grams (sequences of one or more words) between the reconstructed and original text, rewarding partial matches and fluency even when not all tokens are identical. 
\noindent\textbf{ROUGE-L} captures similarity based on the longest common subsequence, highlighting how well the global word order and sentence structure in the reconstruction align with the original text. 

\subsection{Compared Attacks}
\label{app: attacks}
\noindent \textbf{iDLG} (Improved Deep Leakage from Gradients) \cite{zhao2020idlg} is a passive GIA attack that enables recovery of both training sample and label for a single sample. The key idea is that, under cross-entropy loss with one-hot labels, the ground-truth label can be directly inferred from the sign of the gradient of the last FC layer. With the true label identified, iDLG then reconstructs the input sample by iteratively optimizing dummy data to match the observed gradients. 

\noindent \textbf{InvertingGrad (IG)} \cite{geiping2020inverting} is a passive GIA that reconstructs data by optimizing dummy inputs to match the shared gradients. To improve reconstruction quality, it incorporates additional priors (i.e., total variation) to regulate the optimization space. This regularization yields natural-looking images and enhances the fidelity of the recovered data.

\noindent \textbf{GradInversion (GI)}  \cite{yin2021see} is a passive GIA that assumes access to batch normalization (BN) statistics to constrain reconstructions. Following iDLG, it directly infers the ground-truth label from the final layer’s gradients, avoiding unstable label optimization. To further improve reconstruction quality, GradInversion introduces a group fidelity term that iteratively aligns reconstructed gradients with the originals, producing high-resolution and semantically accurate images.

\noindent \textbf{FedLeak} \cite{fan2025boosting} is a passive GIA designed for realistic FL settings. It tackles the core challenge of gradient matching through two techniques: partial gradient matching, which targets informative gradient components, and gradient regularization, which stabilizes optimization.

\noindent \textbf{Fishing} \cite{wen2022fishing} is an active GIA that aims to recovery single sample in a batch of samples. By decreasing the network’s confidence in the target class and target feature, it encourage the gradient come from only the target (single) sample.

\noindent \textbf{Robbing (RtF)} \cite{fowlrobbing} is an active GIA that uses the linear layer leakage to recover training samples. By carefully designing the weights and biases of FC layers, RtF imprints each neuron with a single data point, ensuring that its activation predominantly corresponds to that sample. 

\noindent \textbf{Trap Weight (TW)} \cite{boenisch2023curious} is an active GIA that reconstructs training samples by configuring FC weights so each neuron responds to a single input. It sets roughly half of the weights in FC layer to small negative and half to positive values, isolating individual samples in a batch. TW also uses direct-pass weights in convolutional layers to avoid architectural changes, but it only works for nonnegative inputs; standard normalization with ReLU zeroes negative values, breaking the identity mapping and causing information loss.

\noindent \textbf{LOKI} \cite{zhao2024loki} is an active GIA targeting secure aggregation-based FL, where only aggregated gradients are visible to the server. It inserts an extra convolutional layer before the FC layer and assigns each client a unique subset of kernels with direct-pass method and others are set to zero. This ensures client-specific activations, enabling the server to disentangle and recover individual gradients after aggregation. Building on the imprint method proposed by RtF \cite{fowlrobbing}, it enables individual sample recovery at scale.

\noindent \textbf{Scale-MIA} \cite{shi2023scale} is an active GIA built upon RtF to separate sample contributions in FC layer. It avoids architectural modifications by leveraging the built-in FC layer for linear layer leakage. However, it requires a subset of the training data to train a decoder that maps latent representations back to samples, limiting its generalization to unseen domains.

\subsection{Evaluated Defenses}
\label{app: defenses}
\noindent\textbf{Differential Privacy (DP)} \cite{geyer2017differentially} protects client data by adding random noise to local gradients before they are shared with the server. Each client first clips its gradient to a maximum norm to limit the influence of any individual training sample, and then adds noise, so that the resulting gradient reveals only limited information.

\noindent\textbf{Gradient Quantization} \cite{yue2023gradient} reduces the precision of client gradients before sending them to the server. This limits the amount of information that can be extracted from individual gradients, while also reducing communication overhead. 

\noindent\textbf{Gradient Sparsification} \cite{yue2023gradient} reduces the amount of information transmitted by sending only a subset of the gradient elements to the server, typically those with the largest magnitudes.
This not only lowers communication costs in FL but also limits the information available to potential attackers attempting gradient inversion.

\noindent\textbf{Data augmentation} techniques apply carefully chosen transformations to the training data to prevent adversaries from reconstructing both the augmented and original samples from shared gradients \cite{gao2021privacy, gao2023automatic}. 
To achieve this, ATS \cite{gao2021privacy} employs a privacy score and a training-free accuracy metric to automatically discover effective transformations, yielding a lightweight privacy defense.

\noindent\textbf{Secure aggregation} is a privacy-preserving techniques in FL that allow the server to aggregate local model updates from multiple clients without ever seeing individual updates. A commonly used example is Masking-Based Secure Aggregation (SA) \cite{bonawitz2017practical,fereidooni2021safelearn}. In this approach, each client adds a random mask to its local model update before sending it to the server. When the server sums all masked updates, the masks cancel out, enabling the server to recover only the plaintext of the aggregated model.

\section{Visualization Results}
\label{app: visual}
\begin{figure}
    \centering\includegraphics[width=0.95\linewidth]{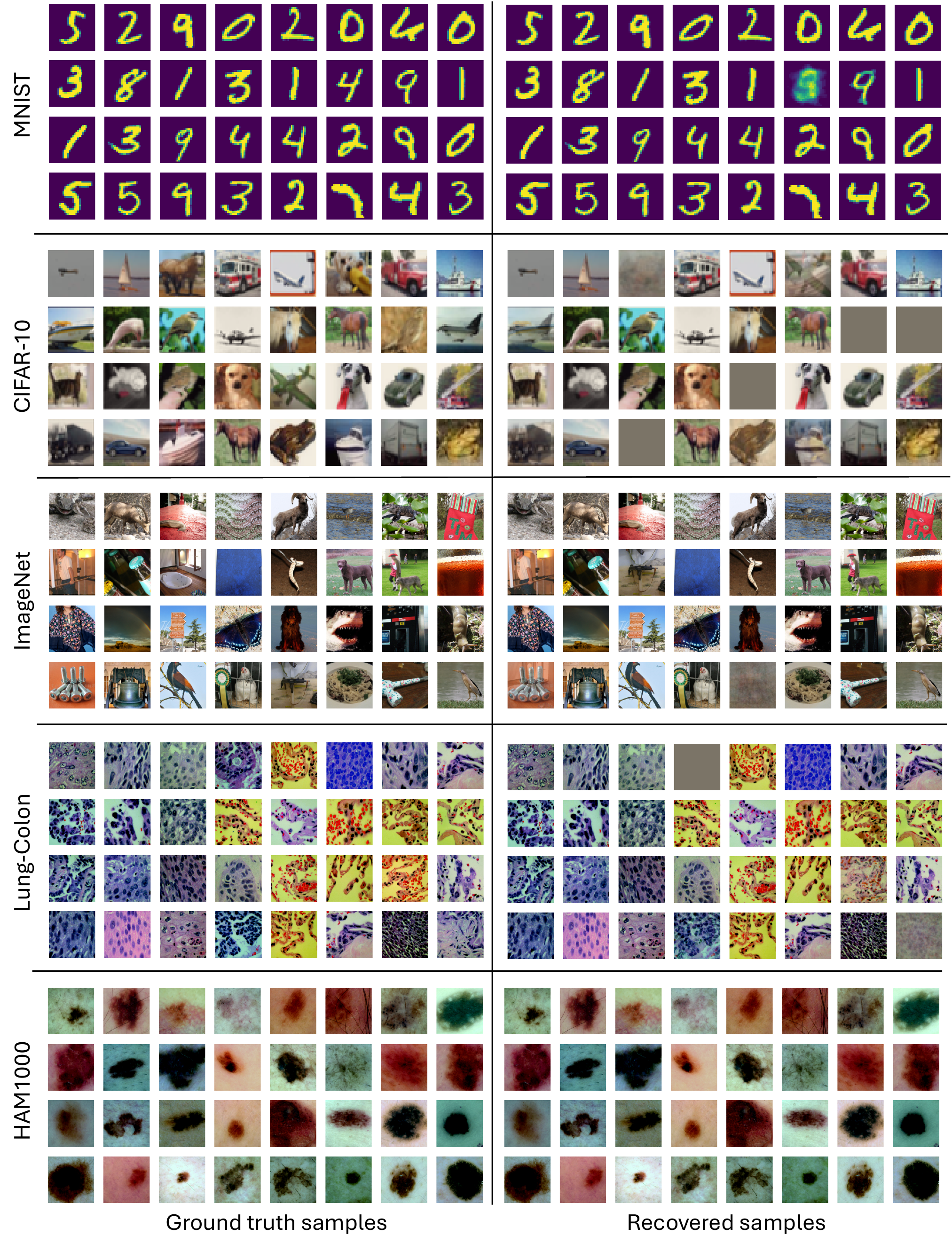}
    \caption{Visual illustration of the recovery effect on five image datasets.}
    \label{fig:visual}
\end{figure}
\begin{figure}
    \centering   \includegraphics[width=0.9\linewidth]{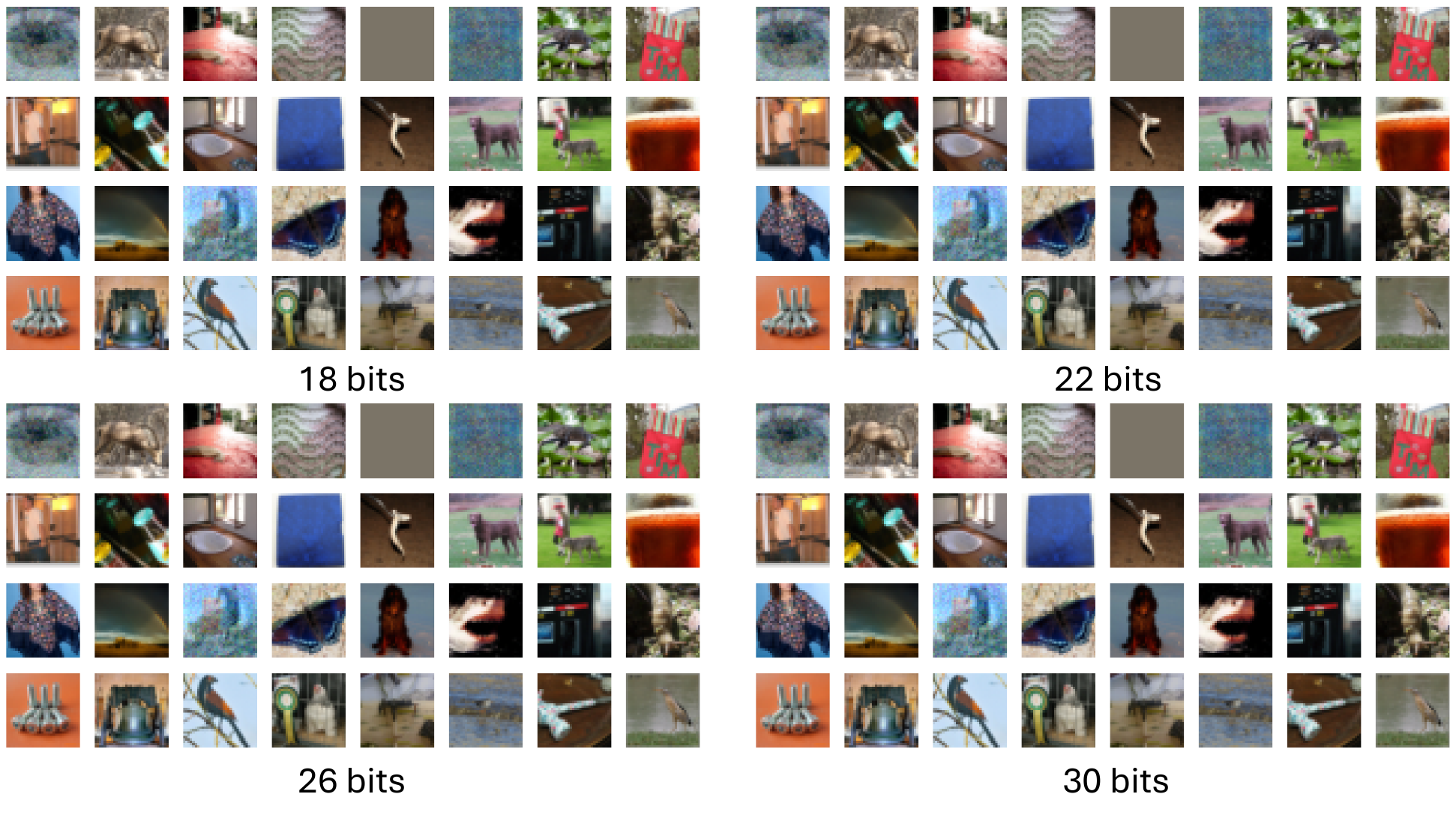}
    \caption{Visualization of attack effect under gradient quantization-based defense.}
    \label{fig:visual_quant}
\end{figure}
\begin{figure}
    \centering   \includegraphics[width=0.9\linewidth]{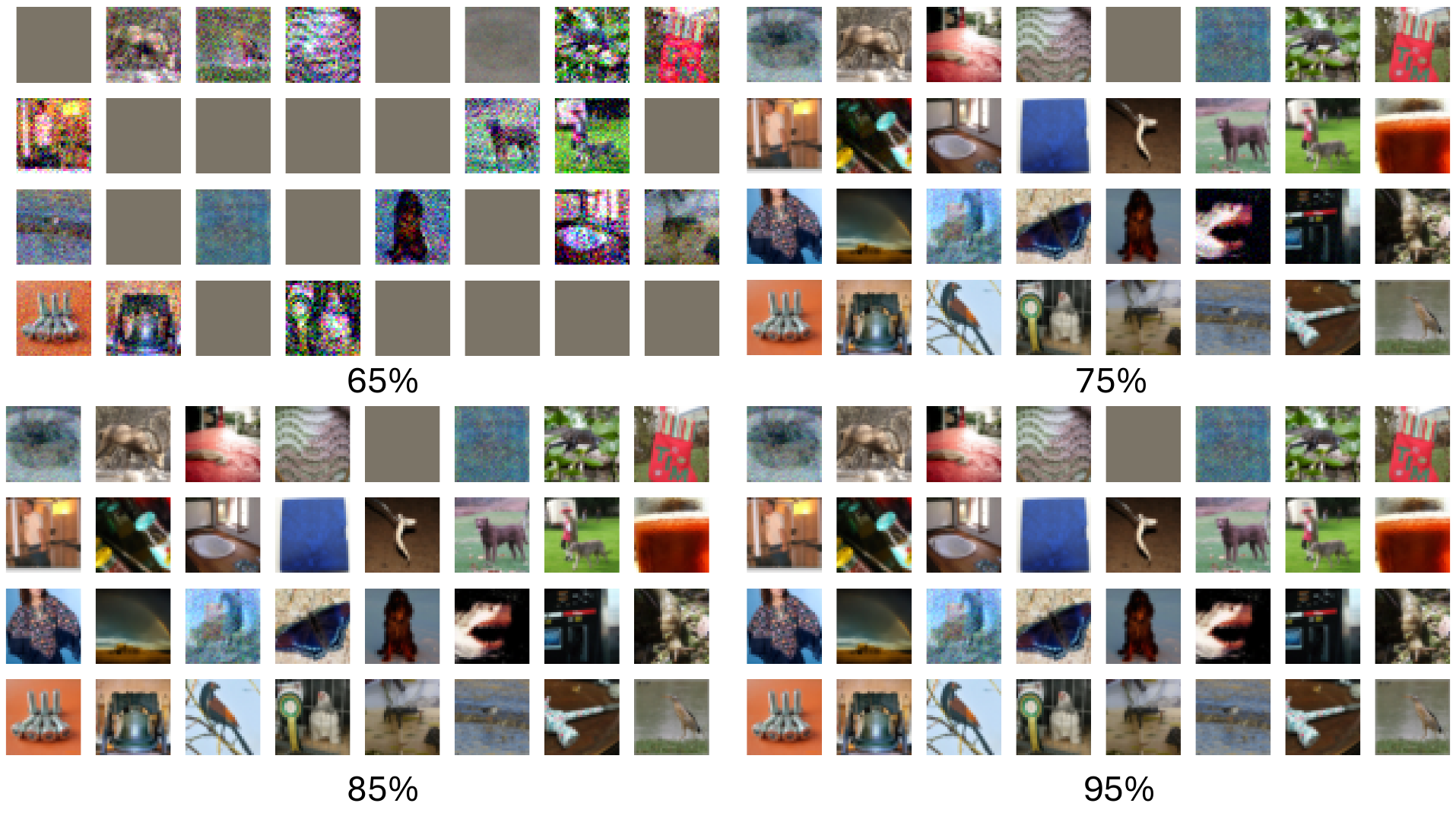}
    \caption{Visualization of attack effect under gradient sparsity-based defense.}
    \label{fig:visual_sparse}
\end{figure}
\begin{figure}
    \centering   \includegraphics[width=0.9\linewidth]{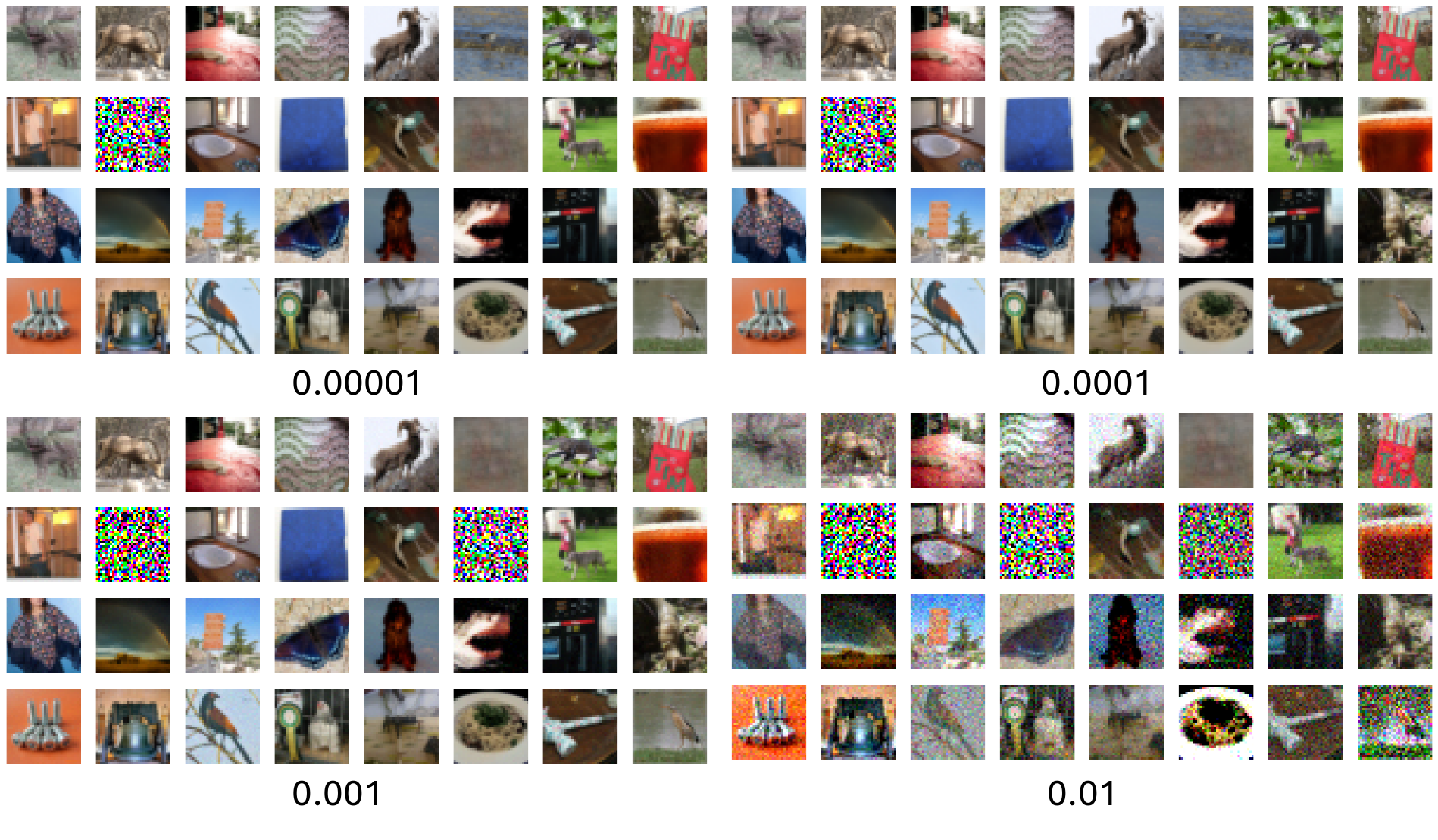}
    \caption{Visualization of attack effect under different differential privacy noise. }
    \label{fig:visual_dp}
\end{figure}
\begin{figure}
    \centering   \includegraphics[width=0.9\linewidth]{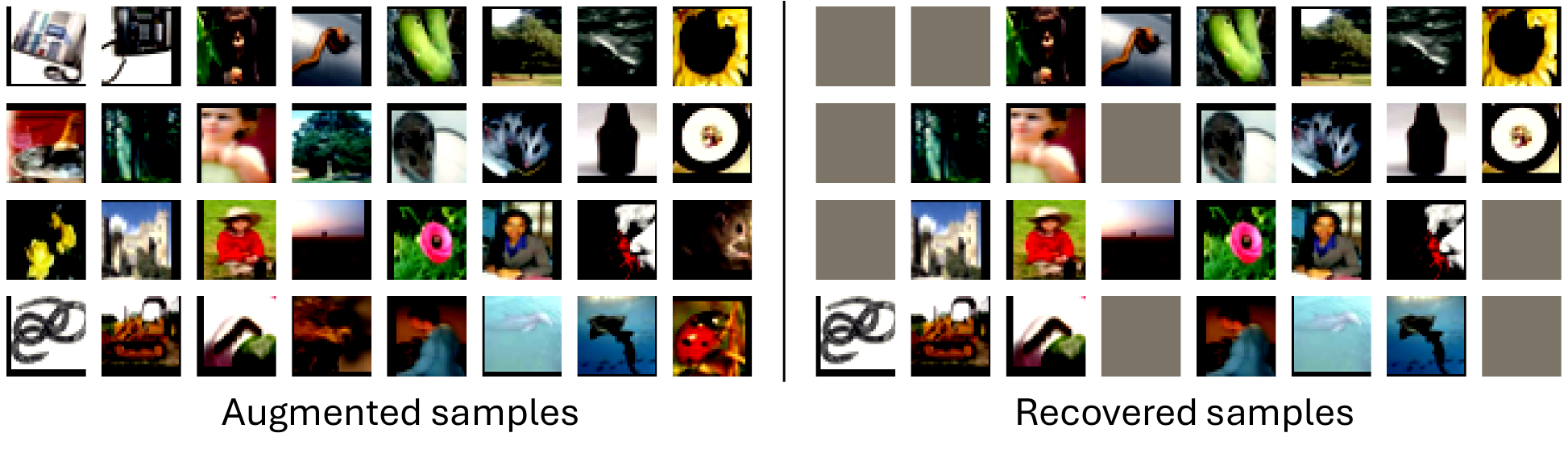}
    \caption{Visualization of the augmented samples and recovered samples.}
    \label{fig:visual_aug}
\end{figure}

Fig. \ref{fig:visual} presents the ground truth and recovered samples across five image datasets using a CNN with a batch size of 32. In all datasets, AERS successfully reconstructs the samples without visually perceptible loss whenever the samples fall into distinct bins.
Fig. \ref{fig:text_visual} shows examples of the original ground-truth training text and the corresponding recovered text on the WikiText dataset using an MLP network.
Fig. \ref{fig:visual_quant} to \ref{fig:visual_aug} shows the recovery effect of our attack under different defenses, including gradient quantization-based defense (Fig. \ref{fig:visual_quant}),  gradient sparsity-based defense (Fig. \ref{fig:visual_sparse}), differential privacy-based defense (Fig. \ref{fig:visual_dp}), and data argumentation-based defense (Fig. \ref{fig:visual_aug}).


\begin{figure}
    \centering\includegraphics[width=0.9\linewidth]{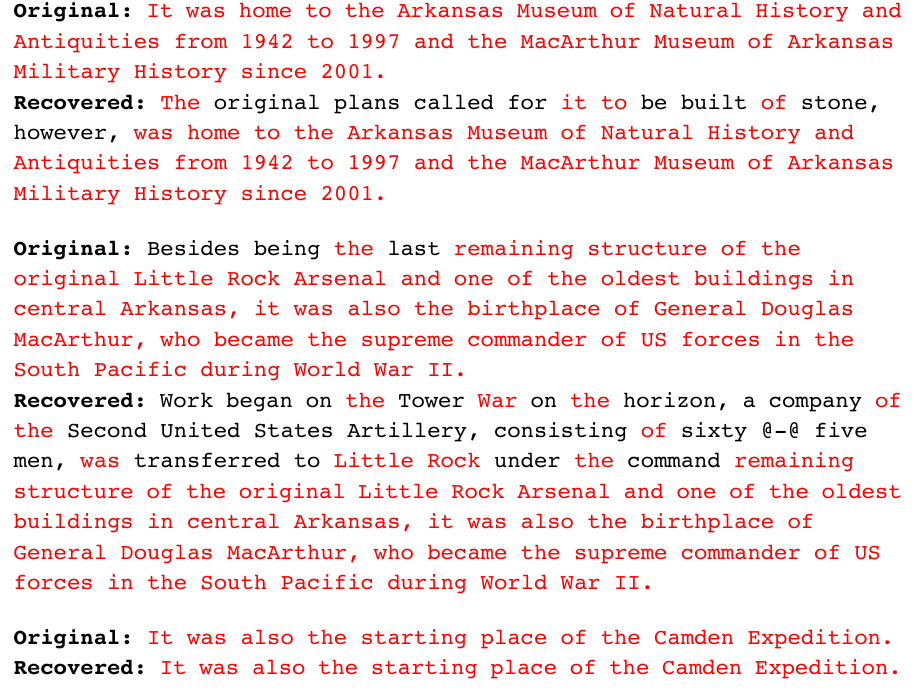}
    \caption{Visual illustration of the recovery effect on the Wikitext dataset. Red text indicates the matching tokens.}
    \label{fig:text_visual}
\end{figure}

\clearpage

\end{appendices}

\end{document}